\documentclass[final,12pt]{sty/colt2023} 
\usepackage[T1]{fontenc}
\usepackage[latin9]{inputenc}
\usepackage{float}
\usepackage{dsfont}
\usepackage{amsmath}
\usepackage{amssymb}

\makeatletter

\floatstyle{ruled}
\newfloat{algorithm}{tbp}{loa}
\providecommand{\algorithmname}{Algorithm}
\floatname{algorithm}{\protect\algorithmname}

\newtheorem{thm}{\protect\theoremname}
\newtheorem{fact}[thm]{\protect\factname}
\newtheorem{lem}[thm]{\protect\lemmaname}
\newtheorem{rem}[thm]{\protect\remarkname}
\ifx\proof\undefined
\newenvironment{proof}[1][\protect\proofname]{\par
	\normalfont\topsep6\p@\@plus6\p@\relax
	\trivlist
	\itemindent\parindent
	\item[\hskip\labelsep\scshape #1]\ignorespaces
}{%
	\endtrivlist\@endpefalse
}
\providecommand{\proofname}{Proof}
\fi

\makeatother

\usepackage{babel}
\providecommand{\factname}{Fact}
\providecommand{\lemmaname}{Lemma}
\providecommand{\remarkname}{Remark}
\providecommand{\theoremname}{Theorem}

\title[Acceleration in Non-Convex Stochastic Optimization with Heavy-Tailed Noise]{Breaking the Lower Bound with (Little) Structure: Acceleration in Non-Convex Stochastic Optimization with Heavy-Tailed Noise}
\usepackage{times}

\coltauthor{\Name{Zijian Liu} \Email{zl3067@stern.nyu.edu}\\
 \Name{Jiawei Zhang}\Email{jzhang@stern.nyu.edu}\\
 \Name{Zhengyuan Zhou}\Email{zzhou@stern.nyu.edu}\\
 \addr Stern School of Business, New York University}

\begin{document}
\global\long\def\E{\mathbb{\mathbb{E}}}%
\global\long\def\F{\mathcal{F}}%
\global\long\def\R{\mathbb{R}}%
\global\long\def\domxi{\mathcal{D}}%
\global\long\def\bzero{\mathbb{\mathbf{\mathbf{0}}}}%
\global\long\def\sgn{\text{sgn}}%
\global\long\def\na{\nabla}%
\global\long\def\indi{\mathds{1}}%

\maketitle
\begin{abstract}
In this paper, we consider the stochastic optimization problem with smooth but not necessarily convex objectives in the heavy-tailed noise regime, where the stochastic gradient's noise is assumed to have bounded $p$th moment ($p\in(1,2]$). This is motivated by a recent plethora of studies in the machine learning literature, which point out that, in comparison to the standard finite-variance assumption, the heavy-tailed noise regime is more appropriate for modern machine learning tasks such as training neural networks. In the heavy-tailed noise regime, \cite{zhang2020adaptive} is the first to prove the $\Omega(T^{\frac{1-p}{3p-2}})$ lower bound for convergence (in expectation) and provides a simple clipping algorithm that matches this optimal rate. Later, \cite{cutkosky2021high} proposes another algorithm, which is shown to achieve the nearly optimal high-probability convergence guarantee $O(\log(T/\delta)T^{\frac{1-p}{3p-2}})$, where $\delta$ is the probability of failure. However, this desirable guarantee is only established under the additional assumption that the stochastic gradient itself is bounded in $p$th moment, which fails to hold even for quadratic objectives and centered Gaussian noise. 

In this work, we first improve the analysis of the algorithm in \cite{cutkosky2021high} to obtain the same nearly optimal high-probability convergence rate $O(\log(T/\delta)T^{\frac{1-p}{3p-2}})$, without the above-mentioned restrictive assumption.
Next, and curiously, we show that one can achieve a faster rate than that dictated by the lower bound $\Omega(T^{\frac{1-p}{3p-2}})$ with only a tiny bit of structure, i.e., when the objective function $F(x)$ is assumed to be in the form of $\E_{\Xi\sim\domxi}[f(x,\Xi)]$, arguably the most widely applicable class of stochastic optimization problems. For this class of problems, we propose the first variance-reduced accelerated algorithm and establish that it guarantees a high-probability convergence rate of $O(\log(T/\delta)T^{\frac{1-p}{2p-1}})$ under a mild condition, which is faster than $\Omega(T^{\frac{1-p}{3p-2}})$. Notably, even when specialized to the standard finite-variance case ($p =2$), our result yields the (near-)optimal high-probability rate $O(\log(T/\delta)T^{-1/3})$, which is unknown before.
\end{abstract}

\begin{keywords}%
  Non-convex stochastic optimization, variance reduction, accelerated algorithm
\end{keywords}

\section{Introduction}

In this paper, we consider the optimization problem with the objective
function $F(x):\R^{d}\to\R$, where $F(x)$ is smooth but not necessarily
convex. With a gradient oracle (i.e. $\na F(x)$
accessible at every $x\in\R^{d}$), it is well-known that Gradient
Descent (GD) algorithm converges in the rate of $\Theta(T^{-1/2})$ after $T$
iterations for finding the critical point. However, in practical scenarios,
(at best) an unbiased noisy estimate $\widehat{\na}F(x)$ ($\E[\widehat{\na}F(x)\vert x]=\na F(x)$)
is available. In such cases, the gold standard is Stochastic Gradient Descent (SGD), a classical
first-order method that has been widely deployed for modern machine learning tasks (such as training deep neural networks). 
Motivated by its empirical success, \cite{ghadimi2013stochastic} has characterized its theoretical guarantees for non-convex objectives and is the
first to establish that SGD converges in expectation with the rate
of $O(T^{-1/4})$ for minimizing the gradient norm under the 
standard finite-variance assumption.
The rate $O(T^{-1/4})$ is also shown to be optimal without any additional assumptions \citep{arjevani2019lower}. 

However, recent studies~\citep{simsekli2019tail,csimcsekli2019heavy,zhang2020adaptive} point out that assuming finite-variance noise
is too optimistic for modern machine learning tasks (in particular,
for training neural networks) and
it is more appropriate to assume that the noise only has bounded $p$th moment, i.e., $\E[\|\widehat{\na}F(x)-\na F(x)\|^{p}\vert x]\leq\sigma^{p}$ for some $p\in(1,2]$, which is known as the heavy-tailed regime. This brings significant challenges in algorithmic and theoretical fronts, as SGD may fail
to converge and the existing theory for SGD becomes invalid when $p\neq 2$. 

Recently, exciting new developments have made progress in overcoming these challenges. In particular, \cite{zhang2020adaptive, cutkosky2021high} propose two new provable algorithms, both of which employ the clipping gradient method.
More precisely, given a stochastic gradient $\widehat{\na}F(x)$,
they both consider a new truncated random variable $g=\min\left\{ 1,\frac{M}{\|\widehat{\na}F(x)\|}\right\} \widehat{\na}F(x)$
where $M$ is the clipping magnitude, in replacement of using $\widehat{\na}F(x)$
directly. As long as $M$ is picked appropriately, \cite{zhang2020adaptive}
shows a convergence rate (in expectation) of $O(T^{\frac{1-p}{3p-2}})$
for SGD combined with clipping directly, which matches the
(in-expectation) lower bound $\Omega(T^{\frac{1-p}{3p-2}})$ proved in the same paper. 
A step further,
\cite{cutkosky2021high} studies the high-probability convergence
behavior, which provides a stronger guarantee for each individual run.
Their (modified) SGD with clipping is shown to converge at
a 
rate of $O(\log(T/\delta)T^{\frac{1-p}{3p-2}})$ with probability
at least $1-\delta$. Unfortunately, the result in \cite{cutkosky2021high}
requires that the stochastic
gradients are bounded in $p$th moment: $\E[\|\widehat{\na}F(x)\|^{p}\vert x]\leq G^{p}$
for some $G>0$, which is
rather restrictive since it can not even hold when $F(x)$ is considered
as quadratic and $\widehat{\na}F(x)-\na F(x)$ is an independent centered
Gaussian random variable (note that the tail of the Gaussian noise
is very light). As such,  when only considering the standard bounded $p$th moment noise assumption, 
we are naturally led to the following question:
\begin{center}
\textit{Q1: Is it possible to design an algorithm with a provable
high-probability convergence guarantee that (nearly) matches the lower
bound $\Omega(T^{\frac{1-p}{3p-2}})$ for the general class of problems?}
\par\end{center}

Moving beyond the general stochastic optimization problem, 
a particular -- but still general enough -- subclass of problems that are of special interest are $F(x)=\E_{\Xi\sim\domxi}[f(x,\Xi)]$,
where $f(x,\Xi)$ is assumed to be differentiable with respect to
$x$ for every realization of $\Xi$ drawn from a (possibly unknown)
probability distribution $\domxi$. In general, $\na f(x,\Xi)$
is an unbiased estimator of $\na F(x)$: $\E_{\Xi\sim\domxi}[\na f(x,\Xi)\vert x]=\na F(x)$. This structure has attracted significant attention from the optimization
community as many modern machine learning problems can be formulated
in such a form. A recent breakthrough to improve the performance
of algorithms for solving this class of problems is to add the variance
reduction, which is shown to achieve acceleration. More specifically, under the finite-variance  ($p=2$) and additional averaged
smoothness assumptions (i.e. $\E_{\Xi\sim\domxi}[\|\na f(x,\Xi)-\na f(y,\Xi)\|^{2}]\leq L^{2}\|x-y\|^{2}$,
$\forall x,y\in\R^{d}$), \cite{fang2018spider,cutkosky2019momentum,tran2019hybrid,liu2020optimal,li2021page}
propose different algorithms with the same convergence rate of $O(T^{-1/3})$
in expectation, which matches the lower bound in~\cite{arjevani2019lower} and is also faster than the generic $\Theta(T^{-1/4})$ for SGD.

However, it still remains open whether -- and if so, how -- similar acceleration can be achieved in the presence of heavy-tailed noise when $p\in(1,2)$. If the convergence rate as a function of $p$ is continuous (a big ``if" that by no means holds \textit{a priori}), then one would have some hope to do better than  the general lower bound $\Omega(T^{\frac{1-p}{3p-2}})$, since when $p=2$, $\Omega(T^{\frac{1-p}{3p-2}})$ yields $\Omega(T^{-1/4})$, which we know can be improved to $\Theta(T^{-1/3})$ in the subclass. Whereas, it is unclear which acceleration scheme -- if any -- would be effective in the heavy-tailed noise setting (under the subclass of the problems of the form $F(x)=\E_{\Xi\sim\domxi}[f(x,\Xi)]$), thereby leading to the second main question:
\begin{center}
\textit{Q2: Is it possible to find an algorithm with a provable convergence
guarantee that outperforms the general lower bound $\Omega(T^{\frac{1-p}{3p-2}})$ and (nearly) matches the optimal $\Theta(T^{-1/3})$ rate when $p=2$?}
\par\end{center}

\subsection{Our Contributions}

We provide affirmative answers to both questions. 
For \textit{Q1}, surprisingly, the algorithm in \cite{cutkosky2021high}
without any modification is enough: under an improved
analysis, the additional assumption of bounded $p$th moment stochastic
gradients can be removed. 
To do so,  we revisit the algorithm Normalized SGD with Clipping and
Momentum, proposed in \cite{cutkosky2021high} and improve the analysis
by employing the proof idea from \cite{gorbunov2020stochastic} to
obtain a better result. To be more precise, \textit{without} the assumption
of bounded $p$th moment stochastic gradients and therefore for the general class of non-convex stochastic optimization problems, we show that the algorithm can converge at the rate of $O(\log(T/\delta)T^{\frac{1-p}{3p-2}})$ with probability at least $1-\delta$. 

For \textit{Q2}, we provide a new and \textit{the
first} \textit{accelerated} algorithm for the heavy-tailed noise setting
and establish the convergence rate of $O(\log(T/\delta)T^{\frac{1-p}{2p-1}})$
 with probability at least $1-\delta$ under a mild condition, which is faster than the general lower bound \textit{$\Omega(T^{\frac{1-p}{3p-2}})$}
and reduces to the nearly optimal rate $O(\log(T/\delta)T^{-1/3})$
when $p=2$.
Our algorithm is designed by integrating
a new variant of the variance-reduced gradient estimator \citep{cutkosky2019momentum, tran2019hybrid, liu2020optimal}
into the Normalized SGD with Clipping and Momentum algorithm.
To the best of our knowledge, this is the first algorithm provably
guaranteeing a faster convergence rate compared with the existing
lower bound $\Omega(T^{\frac{1-p}{3p-2}})$ that is proved for
the general heavy-tailed noise problem.
When specialized to $p=2$, our result yields the nearly optimal high-probability $\widetilde{\Theta}(T^{-1/3})$ rate for the standard finite-variance setting, thereby improving the existing state of knowledge where only in-expectation bound (of the same rate) is known. 

\subsection{Related Work}

\textbf{Convergence with heavy-tailed noise: }When the noise is assumed
to have the finite $p$th moment for $p\in(1,2]$, \cite{zhang2020adaptive}
is the first to prove an $O(T^{\frac{1-p}{3p-2}})$ convergence rate
in expectation by combining SGD and clipping directly. Later, \cite{cutkosky2021high}
proposes a novel algorithm, Normalized SGD with Clipping and Momentum, which enjoys
the provable high-probability convergence behavior attaining the rate
of $O(\log(T/\delta)T^{\frac{1-p}{3p-2}})$ with probability at least
$1-\delta$ after $T$ iterations running. Compared with\textbf{ }\cite{zhang2020adaptive},
the convergence rate is almost the same up to an extra logarithmic
factor, whereas, \cite{cutkosky2021high} requires the additional
restrictive assumption of bounded $p$th moment gradient estimators.
\cite{gorbunov2020stochastic} is the first to prove the high-probability
bounds of clipping algorithms for smooth convex optimization on $\R^{d}$
when $p=2$. We extend their proof idea in this work.

\textbf{Lower bound with heavy-tailed noise: }As far as we know, the
only existing lower bound when considering the heavy-tailed noise
appears in \cite{zhang2020adaptive}, in which the authors prove that
the convergence rate of any algorithm for finding the critical point
can not exceed $\Omega(T^{\frac{1-p}{3p-2}})$ when the objective
function is assumed to be smooth. This means that the results in \cite{zhang2020adaptive,cutkosky2021high}
are both (nearly) optimal. However, if $F(x)$ admits the special
structure $F(x)=\E_{\Xi\sim\domxi}[f(x,\Xi)]$ and satisfies the averaged
smoothness property, the lower bound $\Omega(T^{-1/4})$ is not tight
anymore when $p=2$. \cite{arjevani2019lower} provides an improved
result $\Omega(T^{-1/3})$ for this special case. But if $p$ is considered
as strictly smaller than $2$, whether a tighter lower bound exists
or not remains unknown.

\textbf{Variance reduction for stochastic optimization:} \cite{roux2012stochastic,johnson2013accelerating,shalev2013stochastic,mairal2013optimization,defazio2014saga}
first introduce the variance reduction technique to speed up the convergence
when the objective function is convex and defined in the finite-sum
form. After the important intermediate work of \cite{allen2017katyusha},
\cite{lan2019unified,zhou2019direct,song2020variance,liu2022adaptive,carmon2022recapp}
propose different algorithms provably attaining the near-optimal or
optimal convergence rate under different situations. For non-convex
problems, the variance reduction technique also has been proved to
improve the convergence rate in different settings. When $F(x)=\E_{\Xi\sim\domxi}[f(x,\Xi)]$,
a large number of works \citep{fang2018spider, cutkosky2019momentum, tran2019hybrid, liu2020optimal, li2021page}
prove the $O(T^{-1/3})$ convergence rate in expectation for new algorithms,
which improves upon the well-known speed of $\Theta(T^{-1/4})$ for
the vanilla SGD or momentum SGD and matches the lower bound of $\Omega(T^{-1/3})$
\citep{arjevani2019lower} under the averaged smoothness assumption.
However, to the best of our knowledge, no results have been established
for heavy-tailed noises among all existing works related to variance
reduction.

\section{Preliminaries}

\textbf{Notations}: Let $\left[d\right]$ denote the set $\left\{ 1,2,\cdots,d\right\} $
for any integer $d\geq1$. $\langle\cdot,\cdot\rangle$ is the standard
Euclidean inner product on $\R^{d}$. $\left\Vert \cdot\right\Vert $
represents $\ell_{2}$ norm. $a\land b$ and $a\lor b$ are defined
as $\min\left\{ a,b\right\} $ and $\max\left\{ a,b\right\} $ respectively.
$\sgn(x)$ indicates the sign function satisfying $\sgn(x)=1$ for
$x\geq0$ and $-1$ otherwise.

We focus on the following two non-convex optimization problems in
this work, 

\textbf{P1}: We consider the following problem 
\begin{equation}
\min_{x\in\R^{d}}F(x)\label{eq:P1}
\end{equation}
where the function $F(x)$ is only assumed to be differentiable on
$\R^{d}$ but without any special structure.

\textbf{P2}: In this problem, our objective function is chosen to
have the following special form
\begin{equation}
\min_{x\in\R^{d}}F(x)=\E_{\Xi\sim\domxi}\left[f(x,\Xi)\right]\label{eq:P2}
\end{equation}
where $\Xi$ obeys a (possibly unknown) probability distribution $\domxi$.
We will omit the writing of the distribution $\domxi$ for simplicity
in the remaining paper. In this case, we assume that both $F(x)$
and $f(x,\Xi)$ are differentiable with respect to any $x$ on $\R^{d}$.
$\na f(x,\Xi)$ denotes the gradient taken on $x$ for any realization
$\Xi$ drawn from the distribution $\domxi$.

We note that P1 and P2 can cover most non-convex stochastic optimization
problems, hence which are very general. Additionally, our analysis
relies on the following assumptions.

\textbf{(1) Finite lower bound}: $F_*= \inf_{x\in\R^d}F(x)>-\infty$.

\textbf{(2) Unbiased gradient estimator}: We are able to access a
history-independent, unbiased gradient estimator for both P1 and P2.
More specifically, for P1, a stochastic gradient estimator $\widehat{\na}F$
satisfying $\E[\widehat{\na}F(x)\vert x]=\na F(x),\forall x\in\R^{d}$
is provided; for P2, we are able to draw independent $\Xi$ from the
distribution $\domxi$ and compute $\na f(x,\Xi)$ satisfying $\E[\na f(x,\Xi)\vert x]=\na F(x)$.

\textbf{(3) Bounded $p$th moment noise}: There exist $p\in\left(1,2\right]$
and $\sigma\geq0$ denoting the noise level such that $\E[\|\widehat{\na}F(x)-\na F(x)\|^{p}\vert x]\leq\sigma^{p}$
for P1 and $\E[\|\na f(x,\Xi)-\na F(x)\|^{p}\vert x]\leq\sigma^{p}$
for P2.

\textbf{(4) $L$-smoothness}: $\exists L>0$ such that $\forall x,y\in\R^{d}$,
$\|\na F(x)-\na F(y)\|\leq L\|x-y\|$ for P1 and $\|\na f(x,\Xi)-\na f(y,\Xi)\|\leq L\|x-y\|$
with probability $1$ for P2.

Here we briefly discuss our assumptions. First, Assumptions (1) and
(2) are common in the related literature on stochastic optimization
problems. Assumption (3) is the definition of the heavy-tailed noise,
which includes the widely used finite variance assumption as a subcase
by considering $p=2$. Assumption (4) for P1 is standard for smooth
optimization problems. Though the almost surely smoothness property
for P2 seems stronger, it is realistic in practice and used in lots
of works, e.g., \cite{cutkosky2019momentum,levy2021storm+}. We remark
that Assumption (4) for P2 implies that $F(x)$ itself is also $L$-smooth.
In section \ref{sec:open}, we discuss the limitation of this assumption
for P2. The following two facts are well-known results under our assumptions,
the proof of which can be found in \cite{nesterov2018lectures,lan2020first},
hence, is omitted here.
\begin{fact}
\label{fact:fact-1}Under Assumption (4), we have $F(x)\leq F(y)+\langle\na F(y),x-y\rangle+\frac{L}{2}\left\Vert x-y\right\Vert ^{2}$,
$\forall x,y\in\R^{d}$.
\end{fact}
\begin{fact}
\label{fact:fact-2}Under Assumptions (1) and (4), we have $\left\Vert \na F(x)\right\Vert \leq\sqrt{2L(F(x)-F_*)}$,
$\forall x\in\R^{d}$.
\end{fact}

\section{Algorithms and Results\label{sec: algo}}

In this section, we first state the improved result for Algorithm
\ref{alg:algo-1} proposed by \cite{cutkosky2021high} in Section
\ref{subsec:improved-ashok}. Then we present our new Algorithm \ref{alg:algo-2}
along with its convergence theorem in Section \ref{subsec:new-algo}.
Our algorithm is the first to achieve the accelerated convergence
rate $O(\log(T/\delta)T^{\frac{1-p}{2p-1}})$ beyond the known lower
bound $\Omega(T^{\frac{1-p}{3p-2}})$.

\subsection{Improved Result for the Existing Algorithm\label{subsec:improved-ashok}}

\begin{algorithm}[h]
\caption{\label{alg:algo-1}Normalized SGD with Clipping and Momentum \citep{cutkosky2021high}}

\textbf{Input}: $x_{1}\in\R^{d}$, $0\leq\beta<1$, $M>0$, $\eta>0$.

Set $d_{0}=\bzero$

\textbf{for} $t=1$ \textbf{to} $T$ \textbf{do}

$\quad$$g_{t}=\left(1\land\frac{M}{\left\Vert \widehat{\na}F(x_{t})\right\Vert }\right)\widehat{\na}F(x_{t})$

$\quad$$d_{t}=\beta d_{t-1}+\left(1-\beta\right)g_{t}$

$\quad$$x_{t+1}=x_{t}-\eta\frac{d_{t}}{\left\Vert d_{t}\right\Vert }$

\textbf{end for}
\end{algorithm}

We introduce every part of Algorithm \ref{alg:algo-1} here briefly
and refer the reader to \cite{cutkosky2021high} for details. Algorithm
\ref{alg:algo-1} integrates three main techniques: gradient clipping,
momentum update and normalization. The clipped gradient is to deal
with the heavy-tailed noise issue. Injecting momentum into the (clipped)
stochastic gradient vector can be viewed as to correct update direction.
Normalizing in the update rule allows to significantly simplify the
analysis.

Our improved theoretical result is shown in Theorem \ref{thm:thm-algo-1}.
Unlike \cite{cutkosky2021high}, we no more require the restrictive
assumption of bounded $p$th moment stochastic gradients but obtain
the same convergence rate, which is known to be optimal up to a logarithmic
factor.
\begin{thm}
\label{thm:thm-algo-1}Considering P1 with Assumptions (1), (2), (3)
and (4), let $\Delta_{1}=F(x_{1})-F_*$, then under the following
choices after $T$ iterations running
\begin{align*}
\beta & =1-T^{-\frac{p}{3p-2}};\qquad M=\frac{\sigma}{\left(1-\beta\right)^{1/p}}\lor4\sqrt{L\Delta_{1}};\\
\eta & =\sqrt{\frac{\left(1-\beta\right)\Delta_{1}}{6TL}}\land\frac{1-\beta}{9\beta}\sqrt{\frac{\Delta_{1}}{L}}\land\frac{\Delta_{1}}{120TM\left(1-\beta\right)\log\frac{4T}{\delta}}.
\end{align*}
Algorithm \ref{alg:algo-1} guarantees that with probability at least
$1-\delta$, there is
\[
\frac{1}{T}\sum_{t=1}^{T}\left\Vert \na_{t}\right\Vert =O\left(\frac{\sqrt{L\Delta_{1}}}{T^{\frac{p-1}{3p-2}}}\lor\frac{\sigma\log\frac{T}{\delta}}{T^{\frac{p-1}{3p-2}}}\lor\frac{\sqrt{L\Delta_{1}}\log\frac{T}{\delta}}{T^{\frac{p}{3p-2}}}\right)=O\left(\frac{\log\frac{T}{\delta}}{T^{\frac{p-1}{3p-2}}}\right).
\]
\end{thm}
We first discuss the choices of parameters. The momentum parameter
$\beta$ is chosen essentially the same as in \cite{cutkosky2021high}.
However, the clipping magnitude $M$ is very different because there
no more exists a uniform upper bound $G$ on $\E[\|\widehat{\na}F(x)\|^{p}]^{1/p}$
that is used to decide $M=G/(1-\beta)^{1/p}$ in \cite{cutkosky2021high}.
Intuitively, one can recognize $\sigma$ as a proxy of $G$ in the
current choice of $M$. The appearance of the term $4\sqrt{L\Delta_{1}}$
in $M$ is due to the proof technique. The interested reader could
refer to Section \ref{sec:theory-analysis} for a detailed explanation.
Finally, the step size $\eta$ is chosen by balancing every other
term to get the right convergence rate.

We would like to talk about the convergence guarantee further. At
first glance, the rate seems perfect since it already matches the
lower bound $\Omega(T^{-\frac{1-p}{3p-2}})$ up to a logarithmic factor.
However, the main drawback of this result is lack of adaptivity to
the noise parameter $\sigma$. In other words, when $\sigma=0$, the
best rate we can obtain is still $\widetilde{O}(T^{-1/2}+T^{-1/4})$
by taking $p=2$ (note that $p\in(1,2]$ can be chosen arbitrarily
when $\sigma=0$), which is far from the optimal rate $\Theta(T^{-1/2})$\footnote{We note that the bounds in \cite{cutkosky2021high} suffer the same issue.}.
For now, how to get a (nearly) optimal rate at the same time adapting
to the level of noise $\sigma$ is still unclear to us.

\subsection{The First Accelerated Algorithm with Heavy-Tailed Noise\label{subsec:new-algo}}

\begin{algorithm}[h]
\caption{\label{alg:algo-2}Accelerated Normalized SGD with Clipping and Momentum}

\textbf{Input}: $x_{1}\in\R^{d}$, $0\leq\beta<1$, $M>0$, $\eta>0$.

Set $d_{0}=\bzero$

\textbf{for} $t=1$ \textbf{to} $T$ \textbf{do}

$\quad$sample $\Xi_{t}\sim\domxi$

$\quad$$g_{t}=\left(1\land\frac{M}{\left\Vert \na f(x_{t},\Xi_{t})\right\Vert }\right)\na f(x_{t},\Xi_{t})$

$\quad$$d_{t}=\beta d_{t-1}+\left(1-\beta\right)g_{t}+\indi_{t\geq2}\beta\left(\na f(x_{t},\Xi_{t})-\na f(x_{t-1},\Xi_{t})\right)$

$\quad$$x_{t+1}=x_{t}-\eta\frac{d_{t}}{\left\Vert d_{t}\right\Vert }$

\textbf{end for}
\end{algorithm}

We are now ready to state our new algorithm designed for P2, Accelerated
Normalized SGD with Clipping and Momentum, as shown in Algorithm \ref{alg:algo-2},
the construction of which is mainly inspired by the works \citep{cutkosky2019momentum, tran2019hybrid, liu2020optimal, cutkosky2021high}.
Compared with Algorithm \ref{alg:algo-1}, the key difference is in
how to update the vector $d_{t}$. In the Normalized SGD with Clipping
and Momentum algorithm, $d_{t}$ is defined as
\begin{equation}
d_{t}=\beta d_{t-1}+\left(1-\beta\right)g_{t},\label{eq:momentum}
\end{equation}
which adds momentum to the update simply. Though it is believed that
the momentum part in (\ref{eq:momentum}) can reduce the bias between
$d_{t}$ and the true gradient $\na F(x_{t})$ to accelerate the convergence,
as far as we know, no theoretical justification has been established
for this guess in non-convex optimization problems even considering
the case $g_{t}=\na f(x_{t},\Xi_{t})$.

In contrast, our gradient estimator $d_{t}$ comes from the framework
of momentum-based variance-reduced SGD put forward by \cite{cutkosky2019momentum,tran2019hybrid,liu2020optimal}.
The original template is proposed under the finite variance noise
assumption and is written as follows (consider $t\geq2$ for simplicity)
\begin{equation}
d_{t}=\beta\left(d_{t-1}-\na f(x_{t-1},\Xi_{t})\right)+\na f(x_{t},\Xi_{t}).\label{eq:STORM}
\end{equation}
However, this definition can not be applied to the heavy-tailed noise
case directly. Thanks to the analysis for (\ref{eq:STORM}) in previous
works, we know that part $(ii)$ in (\ref{eq:STORM-re}) (reformulation
of (\ref{eq:STORM})) actually plays a critical role in the effect
of the variance reduction. Part $(i)$ can be thought of as the same
as the traditional momentum update rule.
\begin{equation}
d_{t}=\underbrace{\beta d_{t-1}+\left(1-\beta\right)\na f(x_{t},\Xi_{t})}_{(i)}+\underbrace{\beta\left(\na f(x_{t},\Xi_{t})-\na f(x_{t-1},\Xi_{t})\right)}_{(ii)}.\label{eq:STORM-re}
\end{equation}
Hence the idea for the new definition of $d_{t}=\beta d_{t-1}+\left(1-\beta\right)g_{t}+\beta\left(\na f(x_{t},\Xi_{t})-\na f(x_{t-1},\Xi_{t})\right)$
in Algorithm \ref{alg:algo-2} is natural and clear now, which incorporates
the momentum rule of (\ref{eq:momentum}) used in Algorithm \ref{alg:algo-1}
and part $(ii)$ in (\ref{eq:STORM-re}) to utilize the variance reduction
technique.

Next, we turn to the convergence guarantee of Algorithm \ref{alg:algo-2}
shown in Theorem \ref{thm:thm-algo-2}. One can see the variance reduction
idea indeed works and improves the convergence rate to $O(\log(T/\delta)T^{\frac{1-p}{2p-1}})$,
which is faster than $\Omega(T^{\frac{1-p}{3p-2}})$ when $p\in(1,2]$.
Therefore, our algorithm is the first to achieve the acceleration
in the heavy-tailed noise regime but beyond the general lower bound.
Notably, when $p=2$, the speed reduces to $O(\log(T/\delta)T^{-1/3})$
nearly matching the lower bound $\Omega(T^{-1/3})$ for P2.
\begin{thm}
\label{thm:thm-algo-2}Considering P2 with Assumptions (1), (2), (3)
and (4), let $\Delta_{1}=F(x_{1})-F_*$, then under the following
choices after $T$ iterations running
\begin{align*}
\beta & =1-T^{-\frac{p}{2p-1}};\qquad M=\frac{\sigma}{\left(1-\beta\right)^{1/p}}\lor4\sqrt{L\Delta_{1}};\\
\eta & =\sqrt{\frac{\sqrt{1-\beta}\Delta_{1}}{60TL\log\frac{4T}{\delta}}}\land\frac{1-\beta}{9\beta}\sqrt{\frac{\Delta_{1}}{L}}\land\frac{\Delta_{1}}{120TM\left(1-\beta\right)\log\frac{4T}{\delta}}.
\end{align*}
Algorithm \ref{alg:algo-2} guarantees that with probability at least
$1-2\delta$, there is
\[
\frac{1}{T}\sum_{t=1}^{T}\left\Vert \na_{t}\right\Vert =O\left(\frac{\sqrt{L\Delta_{1}\log\frac{T}{\delta}}}{T^{\frac{3p-2}{4(2p-1)}}}\lor\frac{\sqrt{L\Delta_{1}}}{T^{\frac{p-1}{2p-1}}}\lor\frac{\sigma\log\frac{T}{\delta}}{T^{\frac{p-1}{2p-1}}}\lor\frac{\sqrt{L\Delta_{1}}\log\frac{T}{\delta}}{T^{\frac{p}{2p-1}}}\right)=O\left(\frac{\log\frac{T}{\delta}}{T^{\frac{p-1}{2p-1}}}\right).
\]
\end{thm}
Finally, let us discuss Theorem \ref{thm:thm-algo-2} a bit. First,
the probability $1-2\delta$ is only chosen to simplify the proof,
which can be replaced by $1-\delta$ via changing every $\delta$
to $\delta/2$ in the parameters. Second, $M$ still keeps the same
as it Theorem \ref{thm:thm-algo-1}. In contrast, the choices of $\beta$
and $\eta$ are different and more important to accelerate the convergence.
In particular, the order of $T$ in $\beta$ should be chosen carefully.
Besides, whether the rate is tight or not is unknown for $p\in(1,2)$.
Lastly, we need to mention that Theorem \ref{thm:thm-algo-2} admits
the same flaw of losing adaptivity to the noise $\sigma$ as Theorem
\ref{thm:thm-algo-1}.

\section{Theoretical Analysis\label{sec:theory-analysis}}

We present the ideas for proving Theorems \ref{thm:thm-algo-1} and
\ref{thm:thm-algo-2} here and state some important lemmas, the omitted
proofs of which are provided in Section \ref{sec:app-Missing-Proofs}.
The proof of Theorem \ref{thm:thm-algo-2} is delivered in the last
part of this section. We defer the proof of Theorem \ref{thm:thm-algo-1}
to Section \ref{sec:app-proof-theory} in the appendix.

Our technique contributions can be summarized as follows. We first
extend the ideas used in \cite{gorbunov2020stochastic} as mentioned,
which allows us to forgo the extra assumption in \cite{cutkosky2021high}.
Due to this, both theories only rely on the standard heavy-tailed
assumption. Besides, we modify the proof framework in \cite{cutkosky2021high}
to make it compatible with \cite{gorbunov2020stochastic} and give
an almost unified analysis for Algorithms \ref{alg:algo-1} and \ref{alg:algo-2}
together. Additionally, the effect of variance reduction in $d_{t}$
in Algorithm \ref{alg:algo-2} is quantified in a high-probability
manner, in contrast, which is generally measured by an in-expectation
bound previously. 
Lastly, we would like to emphasize that our proof technique can be easily extended to the same $(2, C\geq 1)$-smooth Banach space used in \cite{cutkosky2021high}. However, we stick to $\R^d$ in this paper for simplicity.

Before diving into the proof, we first outline the key thoughts used
in the analysis. The goal is to show that the event 
\begin{equation}
E_{\tau}=\left\{ \eta\sum_{s=1}^{t}\|\na F(x_{s})\|+F(x_{t+1})-F_*\leq2(F(x_{1})-F_*),\forall t\le\tau\right\} \label{eq:E_tau}
\end{equation}
holds with high probability for every time $0\leq\tau\leq T$. We
remark that \cite{gorbunov2020stochastic} aims to prove the event
(simplified) $\left\{ \eta\sum_{s=1}^{t}F(x_{s})-F(x_{*})+\|x_{t+1}-x_{*}\|^{2}\leq2\|x_{1}-x_{*}\|^{2},\forall t\le\tau\right\} $
happens 
where $x_*$ is the optimal point in the domain.
Though the idea is similar from an abstract level, things
that need to be proved are very different, which is because we are
in a non-convex world.
More precisely, we use the initial function value gap rather than the initial distance to the optimal point to bound other terms. 
An immediate corollary from $E_{\tau}$ is
that $F(x_{t+1})-F_*\leq2(F(x_{1})-F_*)$, which implies
$\|\na F(x_{t})\|\leq2\sqrt{L(F(x_{1})-F_*)}$ for every $t\leq\tau+1$
in a high probability. Now the insight is that $\|\na F(x_{t})\|$
admitting a uniform upper bound through all iterations is highly possible.
Thus, we can use this potential bound to clip the heavy-tailed stochastic
gradient and drop the additional bounded $p$th moment estimates assumption. 
In comparison, the proof strategy in \cite{cutkosky2021high} is much different, in which the authors simply assume $\E[\|\widehat{\na}F(x)\|^p| x]$ is uniformly upper bounded. We note that this stronger assumption immediately implies that $\|\na F(x)\|$ has a uniform upper bound. However, as described above, the event $E_\tau$ is enough to help us choose the proper clipping magnitude. Hence, we can drop the extra assumption used in \cite{cutkosky2021high}. 

With the above plan, we can start the proof. We first introduce some
notations used in the analysis. Let $\F_{t}$ be the natural filtration
for both algorithms. Under this definition, $x_{t}$ is $\F_{t-1}$
measurable. $\E_{t}\left[\cdot\right]$ is used to denote $\E\left[\cdot\mid\F_{t-1}\right]$
for brevity. We also employ the following definitions:
\begin{align*}
\Delta_{t} & =F(x_{t})-F_*;\enskip\na_{t}=\na F(x_{t});\\
\epsilon_{t} & =g_{t}-\nabla_{t};\enskip\epsilon_{t}^{u}=g_{t}-\E_{t}\left[g_{t}\right];\enskip\epsilon_{t}^{b}=\E_{t}\left[g_{t}\right]-\nabla_{t};\\
\xi_{t} & =\begin{cases}
-\na_{1} & t=0\\
d_{t}-\na_{t} & t\geq1
\end{cases};\\
Z_{t} & =\begin{cases}
\indi_{t\geq2}\left(\na_{t-1}-\na_{t}\right) & \text{For Algorithm \ref{alg:algo-1}}\\
\indi_{t\geq2}\left(\nabla f(x_{t},\Xi_{t})-\nabla f(x_{t-1},\Xi_{t})+\na_{t-1}-\na_{t}\right) & \text{For Algorithm \ref{alg:algo-2}}
\end{cases}.
\end{align*}
The decomposition of $\epsilon_{t}=\epsilon_{t}^{u}+\epsilon_{t}^{b}$
is proposed in \cite{gorbunov2020stochastic}. We borrow it here
and first provide the bounds for $\|\epsilon_{t}^{u}\|$ and $\|\epsilon_{t}^{b}\|$
in Lemma \ref{lem:eps-bound}
\begin{lem}
\label{lem:eps-bound}For both Algorithms \ref{alg:algo-1} and \ref{alg:algo-2},
$\forall t\in\left[T\right]$, we have $\left\Vert \epsilon_{t}^{u}\right\Vert \le2M.$
Besides, if $\left\Vert \nabla_{t}\right\Vert \le M/2$, then there
is
\[
\|\epsilon_{t}^{b}\|\le2\sigma^{p}M^{1-p};\enskip\E_{t}[\|\epsilon_{t}^{u}\|^{2}]\le10\sigma^{p}M^{2-p}.
\]
\end{lem}
\begin{rem}
Similar results appear in \cite{zhang2020adaptive} (for heavy-tailed
noise) and \cite{gorbunov2020stochastic} (for bounded variance) before,
but the constants in ours are mildly tighter. 
\end{rem}
The coefficient $1/2$ in the condition $\|\na_{t}\|\leq M/2$ simply
follows the same choice as prior works. It can be changed to any number
in $(0,1)$ with no essential difference. Recall that when $E_{\tau}$
holds, $\|\na_{t}\|\leq2\sqrt{L\Delta_{1}}$ for any $t\leq\tau+1$.
To satisfy $\left\Vert \nabla_{t}\right\Vert \le M/2$, naturally,
$M$ can be chosen larger than $4\sqrt{L\Delta_{1}}$. This answers
why $4\sqrt{L\Delta_{1}}$ shows up in the choice of $M$ for both
Theorems \ref{thm:thm-algo-1} and \ref{thm:thm-algo-2}.

Our next task is to express $\xi_{t}$ via $\xi_{0}$, $Z_{s\leq t}$
and $\epsilon_{s\leq t}$ as shown in Lemma \ref{lem:represent}.
This representation is common in the related literature and can help
us to measure how fast the difference between $d_{t}$ and $\na_{t}$
can decrease.
\begin{lem}
\label{lem:represent}For both Algorithms \ref{alg:algo-1} and \ref{alg:algo-2},
$\forall t\in\left[T\right]$, we have
\[
\xi_{t}=\beta^{t}\xi_{0}+\beta\left(\sum_{s=1}^{t}\beta^{t-s}Z_{s}\right)+\left(1-\beta\right)\left(\sum_{s=1}^{t}\beta^{t-s}\epsilon_{s}\right).
\]
\end{lem}
Equipped with Lemma \ref{lem:represent}, the term $\eta\sum_{t=1}^{\tau}\left\Vert \na_{t}\right\Vert +\Delta_{\tau+1}$
in (\ref{eq:E_tau}) can be upper bounded as follows.
\begin{lem}
\label{lem:basic-ineq}For both Algorithms \ref{alg:algo-1} and \ref{alg:algo-2},
$\forall\tau\in\left\{ 0\right\} \cup\left[T\right]$, we have
\[
\eta\sum_{t=1}^{\tau}\left\Vert \na_{t}\right\Vert +\Delta_{\tau+1}\leq\Delta_{1}+\frac{\tau\eta^{2}L}{2}+\frac{3\beta\eta\sqrt{L\Delta_{1}}}{1-\beta}\indi_{\tau\geq1}+2\eta\sum_{t=1}^{\tau}\beta\left\Vert \sum_{s=1}^{t}\beta^{t-s}Z_{s}\right\Vert +\left(1-\beta\right)\left\Vert \sum_{s=1}^{t}\beta^{t-s}\epsilon_{s}\right\Vert .
\]
\end{lem}
Lemma \ref{lem:basic-ineq} gives us some hints about the next step.
For the term $\frac{\tau\eta^{2}L}{2}+\frac{3\beta\eta\sqrt{L\Delta_{1}}}{1-\beta}\indi_{\tau\geq1}$,
by carefully choosing $\eta$ and $\beta$, it can be bounded by $O(\Delta_{1})$.
Thus, we only need to care about $\|\sum_{s=1}^{t}\beta^{t-s}Z_{s}\|$
and $\|\sum_{s=1}^{t}\beta^{t-s}\epsilon_{s}\|$. However, when Algorithm
\ref{alg:algo-2} is considered, the term $\|\sum_{s=1}^{t}\beta^{t-s}Z_{s}\|$
is essentially different from it in \cite{cutkosky2021high} since
$Z_{s}$ integrates the variance reduction part. Therefore, we
need to come up with a way to quantify the effect of variance reduction
in a high-probability style as noted before. Besides, the high-probability
bound of $\|\sum_{s=1}^{t}\beta^{t-s}\epsilon_{s}\|$ in \cite{cutkosky2021high}
can not be applied either due to no assumption of the bounded $p$th
moment gradient estimates. So the departure from the existing works
starts from here. 

We first bound the easy one $\|\sum_{s=1}^{t}\beta^{t-s}Z_{s}\|$
in Lemma \ref{lem:Z-algo}. As one can see, the bound for Algorithm
\ref{alg:algo-2} is roughly in the order of $\frac{\eta L}{\sqrt{1-\beta}}$.
The acceleration from variance reduction is achieved by improving
a factor of $\frac{1}{\sqrt{1-\beta}}$ compared with $\frac{\eta L}{1-\beta}$
for Algorithm \ref{alg:algo-1}. We also want to mention that assuming
the almost surely smoothness for P2 (Assumption (4)) is critical in
proving this variance-reduced high-probability bound. We refer the
reader to the proof in Section \ref{sec:app-Missing-Proofs} for more
details.
\begin{lem}
\label{lem:Z-algo}We have that
\begin{itemize}
\item for Algorithm \ref{alg:algo-1}, $\forall t\in\left[T\right]$, we
have $\left\Vert \sum_{s=1}^{t}\beta^{t-s}Z_{s}\right\Vert \leq\frac{\eta L}{1-\beta}.$
\item for Algorithm \ref{alg:algo-2}, $\forall t\in\left[T\right]$, we
have $\left\Vert \sum_{s=1}^{t}\beta^{t-s}Z_{s}\right\Vert \leq\frac{9\eta L}{\sqrt{1-\beta}}\log\frac{3T}{\delta}$
with probability at least $1-\frac{\delta}{T}$.
\end{itemize}
\end{lem}
Now we focus on the other term $\|\sum_{s=1}^{t}\beta^{t-s}\epsilon_{s}\|$
and provide its first bound in Lemma \ref{lem:eps-decomp}.
\begin{lem}
\label{lem:eps-decomp}For both Algorithms \ref{alg:algo-1} and \ref{alg:algo-2},
$\forall t\in\left[T\right]$, we have
\[
\left\Vert \sum_{s=1}^{t}\beta_{1}^{t-s}\epsilon_{s}\right\Vert \leq\left|\sum_{s=1}^{t}U_{s}^{t}\right|+\sqrt{2\left|\sum_{s=1}^{t}R_{s}^{t}\right|}+\sqrt{2\sum_{s=1}^{t}\E_{s}\left[\left\Vert \beta^{t-s}\epsilon_{s}^{u}\right\Vert ^{2}\right]}+\left\Vert \sum_{s=1}^{t}\beta^{t-s}\epsilon_{s}^{b}\right\Vert .
\]
where $U_{s}^{t}$ is a martingale difference sequence satisfying
$\left|U_{s}^{t}\right|\leq\left\Vert \beta^{t-s}\epsilon_{s}^{u}\right\Vert $
and $R_{s}^{t}=\left\Vert \beta^{t-s}\epsilon_{s}^{u}\right\Vert ^{2}-\E_{s}\left[\left\Vert \beta^{t-s}\epsilon_{s}^{u}\right\Vert ^{2}\right]$
is also a martingale difference sequence.
\end{lem}
We explain more here to help the reader to understand this complicated
result better. The first step is to use $\|\sum_{s=1}^{t}\beta^{t-s}\epsilon_{s}\|\leq\|\sum_{s=1}^{t}\beta^{t-s}\epsilon_{s}^{u}\|+\|\sum_{s=1}^{t}\beta^{t-s}\epsilon_{s}\|$,
which is very intuitive since we want to use the bounds on $\E_{s}[\|\epsilon_{s}^{u}\|^{2}]$
and $\|\epsilon_{s}^{b}\|$ shown in Lemma \ref{lem:eps-bound}. Next,
we invoke a technical tool (Lemma \ref{lem:ashok-decomp} in Section
\ref{sec:app-tech-tool}) to get $\|\sum_{s=1}^{t}\beta^{t-s}\epsilon_{s}^{u}\|\leq\left|\sum_{s=1}^{t}U_{s}^{t}\right|+\sqrt{2\sum_{s=1}^{t}\left\Vert \beta^{t-s}\epsilon_{s}^{u}\right\Vert ^{2}}$
where the definition of $U_{s}^{t}$ is given in the proof. Finally,
to let $\E_{s}[\|\epsilon_{s}^{u}\|^{2}]$ appear, we employ $R_{s}^{t}$
to obtain the desired result. More details can be found in the proof
in Section \ref{sec:app-Missing-Proofs}.

Thus, our final task is to bound $\left|\sum_{s=1}^{t}U_{s}^{t}\right|$
and $\left|\sum_{s=1}^{t}R_{s}^{t}\right|$. As stated in Lemma \ref{lem:eps-decomp},
both of them are martingale difference sequences. Therefore we would
like to use some concentration inequality. By observing that both
$U_{s}^{t}$ and $R_{s}^{t}$ are bounded almost surely because of
$\|\epsilon_{s}^{u}\|\leq2M$, we can apply the famous Bernstein Inequality
(Lemma \ref{lem:bern} in the appendix) and obtain that the following
two events (Lemmas \ref{lem:U} and \ref{lem:R}) happen with high
probability.
\begin{lem}
\label{lem:U}For both Algorithms \ref{alg:algo-1} and \ref{alg:algo-2},
$\forall t\in\left[T\right]$, we have $\Pr\left[a_{t}\right]\geq1-\frac{\delta}{2T}$
where
\[
a_{t}=\left\{ \left|\sum_{s=1}^{t}U_{s}^{t}\right|\leq\left(\frac{4}{3}+2\sqrt{\frac{5\left(\sigma/M\right)^{p}}{1-\beta}}\right)M\log\frac{4T}{\delta}\text{ or }\sum_{s=1}^{t}\E_{s}\left[\left(U_{s}^{t}\right)^{2}\right]>\frac{10\sigma^{p}M^{2-p}}{1-\beta}\log\frac{4T}{\delta}\right\} .
\]
\end{lem}
\begin{lem}
\label{lem:R}For both Algorithms \ref{alg:algo-1} and \ref{alg:algo-2},
$\forall t\in\left[T\right]$, we have $\Pr\left[b_{t}\right]\geq1-\frac{\delta}{2T}$
where
\[
b_{t}=\left\{ \left|\sum_{s=1}^{t}R_{s}^{t}\right|\leq\left(\frac{16}{3}+4\sqrt{\frac{5\left(\sigma/M\right)^{p}}{1-\beta}}\right)M^{2}\log\frac{4T}{\delta}\text{ or }\sum_{s=1}^{t}\E_{s}\left[\left(R_{s}^{t}\right)^{2}\right]>\frac{40\sigma^{p}M^{4-p}}{1-\beta}\log\frac{4T}{\delta}\right\} .
\]
\end{lem}
We remark that the terms $\frac{10\sigma^{p}M^{2-p}}{1-\beta}\log\frac{4T}{\delta}$
in Lemma \ref{lem:U} and $\frac{40\sigma^{p}M^{4-p}}{1-\beta}\log\frac{4T}{\delta}$
in \ref{lem:R} are chosen carefully to finally let the inequalities
on the conditional variance fail. Then these two events will degenerate
to the bounds on $\left|\sum_{s=1}^{t}U_{s}^{t}\right|$ and $\left|\sum_{s=1}^{t}R_{s}^{t}\right|$,
which are exactly what we need.

With the above lemmas, we are finally able to prove Theorem \ref{thm:thm-algo-2}.

\begin{proof}[of Theorem \ref{thm:thm-algo-2}]
We will use induction to prove that the event $G_{\tau}=E_{\tau}\cap A_{\tau}\cap B_{\tau}\cap C_{\tau}$
holds with probability at least $1-\frac{2\tau\delta}{T}$ for any
$\tau\in\left\{ 0\right\} \cup\left[T\right]$ where
\[
E_{\tau}=\left\{ \eta\sum_{s=1}^{t}\left\Vert \na_{s}\right\Vert +\Delta_{t+1}\leq2\Delta_{1},\forall t\leq\tau\right\} ;\enskip A_{\tau}=\cap_{t=1}^{\tau}a_{t};\enskip B_{\tau}=\cap_{t=1}^{\tau}b_{t};\enskip C_{\tau}=\cap_{t=1}^{\tau}c_{t}.
\]
Events $a_{t}$ and $b_{t}$ are defined in Lemmas \ref{lem:U} and
\ref{lem:R} respectively. $c_{t}$ is from Lemma \ref{lem:Z-algo}
defined as
\begin{equation}
c_{t}=\left\{ \left\Vert \sum_{s=1}^{t}\beta^{t-s}Z_{s}\right\Vert \leq\frac{9\eta L\log\frac{3T}{\delta}}{\sqrt{1-\beta}}\right\} .\label{eq:c_t}
\end{equation}
Note that $E_{0}=\left\{ \Delta_{1}\leq2\Delta_{1}\right\} $ can
be viewed as the whole probability space. Thus, we have $A_{0}=B_{0}=C_{0}=E_{0}$.
Another useful fact is that $A_{\tau}=A_{\tau-1}\cap a_{\tau}$, $B_{\tau}=B_{\tau-1}\cap b_{\tau}$
and $C_{\tau}=C_{\tau-1}\cap c_{\tau}$.

Now we can start the induction. When $\tau=0$, $G_{0}=E_{0}=\left\{ \Delta_{1}\leq2\Delta_{1}\right\} $
holds with probability $1=1-\frac{2\tau\delta}{T}$. Next, suppose
the induction hypothesis $\Pr\left[G_{\tau-1}\right]\geq1-\frac{2(\tau-1)\delta}{T}$
is true for some $\tau\in\left[T\right]$. We will prove that $\Pr\left[G_{\tau}\right]\geq1-\frac{2\tau\delta}{T}$
. To start with, we consider the following event 
\[
E_{\tau-1}\cap A_{\tau}\cap B_{\tau}\cap C_{\tau}=G_{\tau-1}\cap a_{\tau}\cap b_{\tau}\cap c_{\tau}.
\]
From Lemmas \ref{lem:U}, \ref{lem:R} and \ref{lem:Z-algo}, there
are $\Pr\left[a_{\tau}\right]\geq1-\frac{\delta}{2T}$, $\Pr\left[b_{\tau}\right]\geq1-\frac{\delta}{2T}$
and $\Pr\left[c_{\tau}\right]\geq1-\frac{\delta}{T}$. Combining with
the induction hypothesis, we have $\Pr\left[E_{\tau-1}\cap A_{\tau}\cap B_{\tau}\cap C_{\tau}\right]\geq1-\frac{2\tau\delta}{T}.$ 

Note that under the event $E_{\tau-1}$, there is $\Delta_{t}\leq\eta\sum_{s=1}^{t-1}\left\Vert \na_{s}\right\Vert +\Delta_{t}\leq2\Delta_{1}$,
$\forall t\leq\tau$ which implies $\left\Vert \na_{t}\right\Vert \overset{(a)}{\leq}\sqrt{2L\Delta_{t}}\leq2\sqrt{L\Delta_{1}}\overset{(b)}{\leq}\frac{M}{2}$,
$\forall t\leq\tau$ where $(a)$ is by Fact \ref{fact:fact-2} and
for $(b)$ we use $M\geq4\sqrt{L\Delta_{1}}$. Thus, in addition to
$\|\epsilon_{t}^{u}\|\leq2M$, the following two bounds hold for any
$t\leq\tau$ by Lemma \ref{lem:eps-bound}:
\begin{equation}
\E_{t}[\|\epsilon_{t}^{u}\|^{2}]\leq10\sigma^{p}M^{2-p};\enskip\|\epsilon_{t}^{b}\|\leq2\sigma^{p}M^{1-p}.\label{eq:u-and-b}
\end{equation}
Equipped with (\ref{eq:u-and-b}), we can find that for any $t\leq\tau$
\begin{align}
\sum_{s=1}^{t}\E_{s}\left[\left(U_{s}^{t}\right)^{2}\right]\leq & \sum_{s=1}^{t}\E_{s}\left[\left\Vert \beta^{t-s}\epsilon_{s}^{u}\right\Vert ^{2}\right]\leq\sum_{s=1}^{t}\beta^{2t-2s}\cdot10\sigma^{p}M^{2-p}\leq\frac{10\sigma^{p}M^{2-p}}{1-\beta};\label{eq:E-U-square}\\
\sum_{s=1}^{t}\E_{s}\left[\left(R_{s}^{t}\right)^{2}\right]\leq & \sum_{s=1}^{t}\E_{s}\left[\left\Vert \beta^{t-s}\epsilon_{s}^{u}\right\Vert ^{4}\right]\leq\sum_{s=1}^{t}\beta^{4t-4s}\cdot4M^{2}\cdot10\sigma^{p}M^{2-p}\leq\frac{40\sigma^{p}M^{4-p}}{1-\beta}.\label{eq:E-R-square}
\end{align}
Combining with the definitions of $a_{t}$ and $b_{t}$, (\ref{eq:E-U-square})
and (\ref{eq:E-R-square}) imply that under the event $E_{\tau-1}\cap A_{\tau}\cap B_{\tau}\cap C_{\tau}$,
the following two bounds hold for any $t\leq\tau$,
\begin{equation}
\left|\sum_{s=1}^{t}U_{s}^{t}\right|\leq\left(\frac{4}{3}+2\sqrt{\frac{5\left(\sigma/M\right)^{p}}{1-\beta}}\right)M\log\frac{4T}{\delta};\enskip\left|\sum_{s=1}^{t}R_{s}^{t}\right|\leq\left(\frac{16}{3}+4\sqrt{\frac{5\left(\sigma/M\right)^{p}}{1-\beta}}\right)M^{2}\log\frac{4T}{\delta}.\label{eq:U-and-R}
\end{equation}

Assuming $E_{\tau-1}\cap A_{\tau}\cap B_{\tau}\cap C_{\tau}$ happens,
we invoke Lemma \ref{lem:basic-ineq} for time $\tau$ to get
\begin{align}
\eta\sum_{t=1}^{\tau}\left\Vert \na_{t}\right\Vert +\Delta_{\tau+1}\leq & \Delta_{1}+\frac{\tau\eta^{2}L}{2}+\frac{3\beta\eta\sqrt{L\Delta_{1}}}{1-\beta}+2\eta\sum_{t=1}^{\tau}\beta\left\Vert \sum_{s=1}^{t}\beta^{t-s}Z_{s}\right\Vert +\left(1-\beta\right)\left\Vert \sum_{s=1}^{t}\beta^{t-s}\epsilon_{s}\right\Vert \nonumber \\
\overset{(c)}{\leq} & \Delta_{1}+\frac{\tau\eta^{2}L}{2}+\frac{3\beta\eta\sqrt{L\Delta_{1}}}{1-\beta}+\frac{18\beta\tau\eta^{2}L\log\frac{3T}{\delta}}{\sqrt{1-\beta}}+2\eta\left(1-\beta\right)\sum_{t=1}^{\tau}\left\Vert \sum_{s=1}^{t}\beta^{t-s}\epsilon_{s}\right\Vert \nonumber \\
\leq & \Delta_{1}+\frac{20\tau\eta^{2}L\log\frac{4T}{\delta}}{\sqrt{1-\beta}}+\frac{3\beta\eta\sqrt{L\Delta_{1}}}{1-\beta}+2\eta\left(1-\beta\right)\sum_{t=1}^{\tau}\left\Vert \sum_{s=1}^{t}\beta^{t-s}\epsilon_{s}\right\Vert \nonumber \\
\overset{(d)}{\leq} & \Delta_{1}+\frac{20\tau\eta^{2}L\log\frac{4T}{\delta}}{\sqrt{1-\beta}}+\frac{3\beta\eta\sqrt{L\Delta_{1}}}{1-\beta}+40\tau\eta M\left(1-\beta\right)\log\frac{4T}{\delta}\nonumber \\
\overset{(e)}{\leq} & \Delta_{1}+\frac{\Delta_{1}}{3}+\frac{\Delta_{1}}{3}+\frac{\Delta_{1}}{3}=2\Delta_{1}\label{eq:2-Delta}
\end{align}
where in $(c)$ we use the event $c_{t}$ (see (\ref{eq:c_t})) to
bound $\|\sum_{s=1}^{t}\beta^{t-s}Z_{s}\|$. In $(e)$, we plug in
the choice of $\eta=\sqrt{\frac{\sqrt{1-\beta}\Delta_{1}}{60TL\log\frac{4T}{\delta}}}\land\frac{1-\beta}{9\beta}\sqrt{\frac{\Delta_{1}}{L}}\land\frac{\Delta_{1}}{120TM\left(1-\beta\right)\log\frac{4T}{\delta}}$.
The term $\|\sum_{s=1}^{t}\beta^{t-s}\epsilon_{s}\|$ in $(d)$ is
bounded by first employing Lemma \ref{lem:eps-decomp} to get
\begin{align*}
\left\Vert \sum_{s=1}^{t}\beta^{t-s}\epsilon_{s}\right\Vert \leq & \left|\sum_{s=1}^{t}U_{s}^{t}\right|+\sqrt{2\left|\sum_{s=1}^{t}R_{s}^{t}\right|}+\sqrt{2\sum_{s=1}^{t}\E_{s}\left[\left\Vert \beta^{t-s}\epsilon_{s}^{u}\right\Vert ^{2}\right]}+\left\Vert \sum_{s=1}^{t}\beta^{t-s}\epsilon_{s}^{b}\right\Vert \\
\overset{(f)}{\leq} & M\log\frac{4T}{\delta}\left(\frac{4}{3}+\sqrt{\frac{32}{3}}+\frac{2\left(\sigma/M\right)^{p}}{1-\beta}+4\sqrt{5\frac{\left(\sigma/M\right)^{p}}{1-\beta}}+2\sqrt{\sqrt{20\frac{\left(\sigma/M\right)^{p}}{1-\beta}}}\right)\\
\overset{(g)}{\leq} & 20M\log\frac{4T}{\delta}
\end{align*}
where we use (\ref{eq:u-and-b}) and (\ref{eq:U-and-R}) in $(f)$;
$(g)$ is due to $\frac{\left(\sigma/M\right)^{p}}{1-\beta}\leq1$
by the choice of $M$. Now note that (\ref{eq:2-Delta}) implies $e_{\tau}=\left\{ \eta\sum_{t=1}^{\tau}\left\Vert \na_{t}\right\Vert +\Delta_{\tau+1}\leq2\Delta_{1}\right\} $
is a superset of $E_{\tau-1}\cap A_{\tau}\cap B_{\tau}\cap C_{\tau}$.
Therefore
\[
\Pr\left[G_{\tau}\right]=\Pr\left[E_{\tau-1}\cap e_{\tau}\cap A_{\tau}\cap B_{\tau}\cap C_{\tau}\right]=\Pr\left[E_{\tau-1}\cap A_{\tau}\cap B_{\tau}\cap C_{\tau}\right]\geq1-\frac{2\tau\delta}{T}.
\]
Hence, the induction is completed.

Finally, we know $\Pr\left[E_{T}\right]\geq\Pr\left[G_{T}\right]\geq1-2\delta$
which implies with probability at least $1-2\delta$, $\eta\sum_{t=1}^{T}\left\Vert \na_{t}\right\Vert +\Delta_{T+1}\leq2\Delta_{1}$.
By plugging our choices of $\beta$, $M$ and $\eta$, we conclude
\begin{align*}
\frac{1}{T}\sum_{t=1}^{T}\left\Vert \na_{t}\right\Vert \leq & \frac{2\Delta_{1}}{\eta T}=O\left(\frac{\sqrt{L\Delta_{1}\log\frac{T}{\delta}}}{T^{\frac{3p-2}{4(2p-1)}}}\lor\frac{\sqrt{L\Delta_{1}}}{T^{\frac{p-1}{2p-1}}}\lor\frac{\sigma\log\frac{T}{\delta}}{T^{\frac{p-1}{2p-1}}}\lor\frac{\sqrt{L\Delta_{1}}\log\frac{T}{\delta}}{T^{\frac{p}{2p-1}}}\right).
\end{align*}
\end{proof}

\section{Open Questions\label{sec:open}}

It still remains some limitations in our work and there are many
open problems worth exploring. First of all, we wonder whether our
accelerated rate $O(\log(T/\delta)T^{\frac{1-p}{2p-1}})$ can be improved
further or not when $F(x)=\E[f(x,\Xi)]$. We only know that the rate
will reduce to $O(\log(T/\delta)T^{-1/3})$ when $p=2$ matching the
in-expectation lower bound $\Omega(T^{-1/3})$ up to a logarithmic
factor. However, there is nothing known to us for $p\in(1,2)$. Second,
our accelerated result is proved under the assumption that $f(x,\Xi)$
is smooth almost surely. Whereas, we guess it is possible to relax
it to the averaged smooth assumption, i.e., $\E_{\Xi\sim\domxi}[\|f(x,\Xi)-f(y,\Xi)\|^{2}]\leq L^{2}\|x-y\|^{2}$,
which is the standard assumption used in \cite{arjevani2019lower}
for proving the lower bound when $p=2$. Besides, as mentioned before,
our analysis is not adaptive to the noise level $\sigma$. In other
words, when $\sigma=0$, our rate can not recover the well-known and
optimal rate $\Theta(T^{-1/2})$ for deterministic algorithms. Thus,
we believe it is an interesting task to improve our analysis in a
further step. Additionally, it is still unclear how to remove the
extra term $T$ appearing in $\log(T/\delta)$. Finally, our choices
of parameters heavily rely on the prior knowledge of the problem itself,
which may be hard to know or even estimate in practice. Hence, it
is important and worthful to find parameter-free algorithms that can
achieve the same convergence rate for both two problems considered
in this paper. We leave these questions as the future direction and
look forward to them being addressed.

\section*{Acknowledgments\label{sec:acknowledgments}}
This work is generously supported by the National Science Foundation under the grant CCF-2106508.
Additionally, Zhengyuan Zhou would like to acknowledge New York University's Center for Global Economy and Business faculty research grant during the 2023 -- 2024 year. We are also grateful to the anonymous reviewers for their constructive comments and suggestions.

\clearpage

\bibliography{ref}

\newpage

\appendix
\onecolumn

\section{Technical Tools\label{sec:app-tech-tool}}

In this section, we list some helpful technical results that appeared 
in the previous research, some proof of which will be omitted. The
interested reader can refer to the original work for details.

The first inequality we need is the famous Bernstein Inequality for
martingale difference sequence.
\begin{lem}
\label{lem:bern}(Bernstein Inequality for martingale difference sequence
in \cite{bennett1962probability,dzhaparidze2001bernstein}) Suppose
$X_{t\in\left[T\right]}\in\R$ is a martingale difference sequence
adapted to the filtration $\F_{t\in\left[T\right]}$ satisfying $\left|X_{t}\right|\leq R$
almost surely for some constant $R$. Let $\sigma_{t}^{2}=\E\left[\left|X_{t}\right|{}^{2}\mid\F_{t-1}\right]$,
then for any $a>0$ and $F>0$, there is
\[
\Pr\left[\left|\sum_{t=1}^{T}X_{t}\right|>a\text{ and }\sum_{t=1}^{T}\sigma_{t}^{2}\leq F\right]\leq2\exp\left(-\frac{a^{2}}{2F+2Ra/3}\right).
\]
\end{lem}

The following concentration inequality proved in the general Hilbert
Space is also useful. A similar result has appeared in Lemma 12 in \cite{cutkosky2021high}. However, the term $\sum_{s=1}^{T}\sigma_{s}^{2}$
 is stated as $\sum_{s=1}^{t}\sigma_{s}^{2}$ instead, which is not correct after private communication with the authors of \cite{cutkosky2021high}.
Therefore, we provide the correct version of this dimension-free concentration
inequality here\footnote{The results in \cite{cutkosky2021high} still hold after this correction.} with its proof.
\begin{lem}
\label{lem:ashok-ineq}(Corrected version of Lemma 12 in \cite{cutkosky2021high}) Suppose
$X_{t\in\left[T\right]}$ is a martingale difference sequence adapted
to the filtration $\F_{t\in\left[T\right]}$ in a Hilbert Space satisfying
$\left\Vert X_{t}\right\Vert \leq R$ almost surely for some constant
$R$ and $\E\left[\left\Vert X_{t}\right\Vert {}^{2}\mid\F_{t-1}\right]\leq\sigma_{t}^{2}$
almost surely for some constant $\sigma_{t}^{2}$ . Then with probability
at least $1-\delta$, $\forall t\in\left[T\right]$, there is
\[
\left\Vert \sum_{s=1}^{t}X_{s}\right\Vert \leq3R\log\frac{3}{\delta}+3\sqrt{\sum_{s=1}^{T}\sigma_{s}^{2}\log\frac{3}{\delta}}.
\]
\end{lem}

\begin{proof}
    By Lemma \ref{lem:ashok-decomp} (see below), for any $t\in\left[T\right]$
we have
\[
\left\Vert \sum_{s=1}^{t}X_{s}\right\Vert \leq\left|\sum_{s=1}^{t}M_{s}\right|+\sqrt{\max_{s\in\left[t\right]}\left\Vert X_{s}\right\Vert ^{2}+\sum_{s=1}^{t}\left\Vert X_{s}\right\Vert ^{2}}
\]
where $M_{t}\in\R$ is a martingale difference sequence satisfying $\left|M_{t}\right|\leq\left\Vert X_{t}\right\Vert $
almost surely.

By using $\left\Vert X_{t}\right\Vert \leq R$ almost surely, we have
\begin{align*}
\left\Vert \sum_{s=1}^{t}X_{s}\right\Vert \leq & \left|\sum_{s=1}^{t}M_{s}\right|+\sqrt{R^{2}+\sum_{s=1}^{t}\left\Vert X_{s}\right\Vert ^{2}}\\
= & \left|\sum_{s=1}^{t}M_{s}\right|+\sqrt{R^{2}+\sum_{s=1}^{t}\underbrace{\left\Vert X_{s}\right\Vert ^{2}-\E\left[\left\Vert X_{s}\right\Vert ^{2}\mid\F_{s-1}\right]}_{U_{s}}+\sum_{s=1}^{t}\E\left[\left\Vert X_{s}\right\Vert ^{2}\mid\F_{s-1}\right]}\\
\leq & \left|\sum_{s=1}^{t}M_{s}\right|+\sqrt{R^{2}+\sum_{s=1}^{t}U_{s}+\sum_{s=1}^{t}\sigma_{s}^{2}}.
\end{align*}
where the last inequality is due to $\E\left[\left\Vert X_{s}\right\Vert ^{2}\mid\F_{s-1}\right]\leq\sigma_{s}^{2}$
almost surely.

Note that
\[
\left|M_{t}\right|\leq R,\E\left[\left|M_{t}\right|^{2}\mid\F_{t-1}\right]\leq\E\left[\left\Vert X_{t}\right\Vert ^{2}\mid\F_{t-1}\right]\leq\sigma_{t}^{2},
\]
by Freedman's inequality \cite{freedman1975tail}, there is
\[
\Pr\left[\forall t\in\left[T\right],\left|\sum_{s=1}^{t}M_{s}\right|\leq\frac{2R}{3}\log\frac{1}{\delta}+\sqrt{2\sum_{s=1}^{T}\sigma_{s}^{2}\log\frac{1}{\delta}}\right]\geq1-2\delta.
\]
Similarly, we have
\[
\Pr\left[\forall t\in\left[T\right],\sum_{s=1}^{t}U_{s}\leq\frac{2R^{2}}{3}\log\frac{1}{\delta}+\sqrt{2\sum_{s=1}^{T}\sigma_{s}^{2}R^{2}\log\frac{1}{\delta}}\right]\geq1-\delta.
\]
Hence, with probability at least $1-3\delta$, for any $t\in\left[T\right]$
\begin{align*}
\left\Vert \sum_{s=1}^{t}X_{s}\right\Vert \leq & \frac{2R}{3}\log\frac{1}{\delta}+\sqrt{2\sum_{s=1}^{T}\sigma_{s}^{2}\log\frac{1}{\delta}}+\sqrt{R^{2}+\frac{2R^{2}}{3}\log\frac{1}{\delta}+\sqrt{2\sum_{s=1}^{T}\sigma_{s}^{2}R^{2}\log\frac{1}{\delta}}+\sum_{s=1}^{t}\sigma_{s}^{2}}\\
\leq & 3R\max\left\{ 1,\log\frac{1}{\delta}\right\} +3\sqrt{\sum_{s=1}^{T}\sigma_{s}^{2}\max\left\{ 1,\log\frac{1}{\delta}\right\} }.
\end{align*}
By changing $\delta$ to $\delta/3$, the proof is finished.
\end{proof}

The last important tool is also proved by \cite{cutkosky2021high},
the original statement of which is for the Banach Space. We simplify
the result since only $\R^{d}$ is considered in this paper.
\begin{lem}
\label{lem:ashok-decomp}(Lemma 10 in \cite{cutkosky2021high}) Suppose
$X_{t\in\left[T\right]}\in\R^{d}$ is a martingale difference sequence
adapted to the filtration $\F_{t\in\left[T\right]}$. Consider the
sequence of real numbers $Y_{t}$ defined as
\[
Y_{t}=\begin{cases}
0 & t=0\\
\sgn\left(\sum_{i=1}^{t-1}Y_{i}\right)\frac{\langle\sum_{i=1}^{t-1}X_{i},X_{t}\rangle}{\left\Vert \sum_{i=1}^{t-1}X_{i}\right\Vert } & t\neq0\text{ and }\sum_{i=1}^{t-1}X_{i}\neq\bzero\\
0 & t\neq0\text{ and }\sum_{i=1}^{t-1}X_{i}=\bzero
\end{cases}.
\]
Then $Y_{t\in\left[T\right]}$ is a also martingale difference sequence
satisfying $\left|Y_{t}\right|\leq\left\Vert X_{t}\right\Vert ,\forall t\in\left[T\right]$
and
\[
\left\Vert \sum_{t=1}^{T}X_{t}\right\Vert \leq\left|\sum_{t=1}^{T}Y_{t}\right|+\sqrt{\max_{t\in\left[T\right]}\left\Vert X_{t}\right\Vert ^{2}+\sum_{t=1}^{T}\left\Vert X_{t}\right\Vert ^{2}}.
\]
\end{lem}

\section{Missing Proofs In Section \ref{sec:theory-analysis}\label{sec:app-Missing-Proofs}}

In this section, we provide the omitted poofs of lemmas stated in
Section \ref{sec:theory-analysis}. We first help the reader to recall
the notations used in the analysis:
\begin{align*}
\Delta_{t} & =F(x_{t})-F_*;\enskip\na_{t}=\na F(x_{t});\enskip\E_{t}\left[\cdot\right]=\E\left[\cdot\mid\F_{t-1}\right];\\
\epsilon_{t} & =g_{t}-\nabla_{t};\enskip\epsilon_{t}^{u}=g_{t}-\E_{t}\left[g_{t}\right];\enskip\epsilon_{t}^{b}=\E_{t}\left[g_{t}\right]-\nabla_{t};\\
\xi_{t} & =\begin{cases}
-\na_{1} & t=0\\
d_{t}-\na_{t} & t\geq1
\end{cases};\\
Z_{t} & =\begin{cases}
\indi_{t\geq2}\left(\na_{t-1}-\na_{t}\right) & \text{For Algorithm \ref{alg:algo-1}}\\
\indi_{t\geq2}\left(\nabla f(x_{t},\Xi_{t})-\nabla f(x_{t-1},\Xi_{t})+\na_{t-1}-\na_{t}\right) & \text{For Algorithm \ref{alg:algo-2}}
\end{cases}
\end{align*}
where $\F_{t}$ is the natural filtration.

\subsection{Proof of Lemma \ref{lem:eps-bound}}

\begin{proof}
First, from the definition of $\epsilon_{t}^{u}$, we have $\left\Vert \epsilon_{t}^{u}\right\Vert \leq\left\Vert g_{t}\right\Vert +\left\Vert \E_{t}\left[g_{t}\right]\right\Vert \leq\left\Vert g_{t}\right\Vert +\E_{t}\left[\left\Vert g_{t}\right\Vert \right]\leq2M$.
For the second and the third inequalities, we use Algorithm \ref{alg:algo-1}
as an example. The same proof can be applied to Algorithm \ref{alg:algo-2}
directly.

For the second one, we know
\begin{align*}
\left\Vert \epsilon_{t}^{b}\right\Vert  & =\left\Vert \E_{t}\left[g_{t}\right]-\nabla_{t}\right\Vert =\left\Vert \E_{t}\left[g_{t}-\widehat{\na}F(x_{t})\right]\right\Vert \leq\E_{t}\left[\left\Vert g_{t}-\widehat{\na}F(x_{t})\right\Vert \right]\\
 & =\E_{t}\left[\left\Vert \left(\frac{M}{\left\Vert \widehat{\na}F(x_{t})\right\Vert }-1\right)\widehat{\na}F(x_{t})\indi_{\left\Vert \widehat{\na}F(x_{t})\right\Vert \geq M}\right\Vert \right]\\
 & =\E_{t}\left[\left(\left\Vert \widehat{\na}F(x_{t})\right\Vert -M\right)\indi_{\left\Vert \widehat{\na}F(x_{t})\right\Vert \geq M}\right]\\
 & \leq\E_{t}\left[\left(\left\Vert \widehat{\na}F(x_{t})-\na_{t}\right\Vert +\left\Vert \nabla_{t}\right\Vert -M\right)\indi_{\left\Vert \widehat{\na}F(x_{t})\right\Vert \geq M}\right]\\
 & \leq\E_{t}\left[\left\Vert \widehat{\na}F(x_{t})-\na_{t}\right\Vert \indi_{\left\Vert \widehat{\na}F(x_{t})\right\Vert \geq M}\right]
\end{align*}
where we use $\left\Vert \nabla_{t}\right\Vert \leq M/2\leq M$ in
the last step. Note that $\left\Vert \widehat{\na}F(x_{t})\right\Vert \geq M\Rightarrow\left\Vert \widehat{\na}F(x_{t})-\na_{t}\right\Vert \geq M/2$
when $\left\Vert \na_{t}\right\Vert \leq M/2$. Hence, we have
\begin{align*}
 & \E_{t}\left[\left\Vert \widehat{\na}F(x_{t})-\na_{t}\right\Vert \indi_{\left\Vert \widehat{\na}F(x_{t})\right\Vert \geq M}\right]\\
\leq & \E_{t}\left[\left\Vert \widehat{\na}F(x_{t})-\na_{t}\right\Vert \indi_{\left\Vert \widehat{\na}F(x_{t})-\na_{t}\right\Vert \geq M/2}\right]\\
\overset{(a)}{\leq} & \E_{t}\left[\left\Vert \widehat{\na}F(x_{t})-\na_{t}\right\Vert ^{p}\right]^{1/p}\E_{t}\left[\indi_{\left\Vert \widehat{\na}F(x_{t})-\na_{t}\right\Vert \geq M/2}\right]^{1-1/p}\\
\leq & \sigma\Pr\left[\left\Vert \widehat{\na}F(x_{t})-\na_{t}\right\Vert ^{p}\geq\left(M/2\right)^{p}\right]^{1-1/p}\\
\overset{(b)}{\leq} & \sigma\left(\frac{2^{p}\sigma^{p}}{M^{p}}\right)^{1-1/p}=2^{p-1}\sigma^{p}M^{1-p}\leq2\sigma^{p}M^{1-p}
\end{align*}
where $(a)$ is due to Holder Inequality and $(b)$ is because of
Markov's Inequaliy.

For the third inequlaity, we prove it as follows
\begin{align*}
\E_{t}\left[\left\Vert \epsilon_{t}^{u}\right\Vert ^{2}\right] & =\E_{t}\left[\left\Vert g_{t}-\E_{t}\left[g_{t}\right]\right\Vert ^{2}\right]\overset{(c)}{\leq}\E_{t}\left[\left\Vert g_{t}-\na_{t}\right\Vert ^{2}\right]\\
 & =\E_{t}\left[\left\Vert g_{t}-\na_{t}\right\Vert ^{2}\indi_{\left\Vert \widehat{\na}F(x_{t})\right\Vert \geq M}+\left\Vert g_{t}-\na_{t}\right\Vert ^{2}\indi_{\left\Vert \widehat{\na}F(x_{t})\right\Vert <M}\right]\\
 & =\E_{t}\left[\left\Vert \frac{M}{\left\Vert \widehat{\na}F(x_{t})\right\Vert }\widehat{\na}F(x_{t})-\na_{t}\right\Vert ^{2}\indi_{\left\Vert \widehat{\na}F(x_{t})\right\Vert \geq M}+\left\Vert \widehat{\na}F(x_{t})-\na_{t}\right\Vert ^{2}\indi_{\left\Vert \widehat{\na}F(x_{t})\right\Vert <M}\right]\\
 & \overset{(d)}{\leq}\E_{t}\left[\frac{9}{4}M^{2}\indi_{\left\Vert \widehat{\na}F(x_{t})\right\Vert \geq M}+\left(\frac{3}{2}M\right)^{2-p}\left\Vert \widehat{\na}F(x_{t})-\na_{t}\right\Vert ^{p}\right]\\
 & \leq\frac{9}{4}M^{2}\frac{2^{p}\sigma^{p}}{M^{p}}+\left(\frac{3}{2}M\right)^{2-p}\sigma^{p}=\left[9\cdot2^{p-2}+\left(\frac{3}{2}\right)^{2-p}\right]\sigma^{p}M^{2-p}\\
 & \leq10\sigma^{p}M^{2-p}.
\end{align*}
where $(c)$ is due to $\E_{t}\left[\left\Vert g_{t}-\E_{t}\left[g_{t}\right]\right\Vert ^{2}\right]\leq\E_{t}\left[\left\Vert g_{t}-Y\right\Vert ^{2}\right]$
for any $Y\in\F_{t-1}$. $(d)$ is by when $\left\Vert \na_{t}\right\Vert \leq M/2$
there are
\[
\left\Vert \frac{M}{\left\Vert \widehat{\na}F(x_{t})\right\Vert }\widehat{\na}F(x_{t})-\na_{t}\right\Vert \leq M+\left\Vert \na_{t}\right\Vert \leq3M/2
\]
 and 
\[
\left\Vert \widehat{\na}F(x_{t})-\na_{t}\right\Vert \indi_{\left\Vert \widehat{\na}F(x_{t})\right\Vert <M}\leq M+\left\Vert \na_{t}\right\Vert \leq3M/2.
\]
\end{proof}

\subsection{Proof of Lemma \ref{lem:represent}}

\begin{proof}
We first check for Algorithm \ref{alg:algo-1}. Use the definition
of $\epsilon_{t}$, $Z_{t}$ and $\xi_{t}$ here to get for $t\geq2$
\begin{align*}
\xi_{t} & =d_{t}-\nabla_{t}=\beta d_{t-1}+\left(1-\beta\right)g_{t}-\nabla_{t}=\beta\xi_{t-1}+\beta Z_{t}+\left(1-\beta\right)\epsilon_{t}.
\end{align*}
Note that the above equation also holds when $t=1$. Next, we calculate
$\xi_{t}$ for Algorithm \ref{alg:algo-2} by noticing for $t\geq2$
\begin{align*}
\xi_{t} & =d_{t}-\nabla_{t}=\beta d_{t-1}+\left(1-\beta\right)g_{t}+\beta\left(\na f(x_{t},\Xi_{t})-\na f(x_{t-1},\Xi_{t})\right)-\nabla_{t}\\
 & =\beta\xi_{t-1}+\beta Z_{t}+\left(1-\beta\right)\epsilon_{t}.
\end{align*}
This equation is true for $t=1$ again. We use this recursion for
all iterations to finish the proof.
\end{proof}

\subsection{Proof of Lemma \ref{lem:basic-ineq}}

\begin{proof}
We will first prove $\Delta_{t+1}-\Delta_{t}\leq-\eta\left\Vert \na_{t}\right\Vert +2\eta\left\Vert \xi_{t}\right\Vert +\frac{\eta^{2}L}{2}$.
This result has been shown in \cite{cutkosky2021high}. We provide
the analysis below for completeness. Starting with Fact \ref{fact:fact-1}
\begin{align*}
\Delta_{t+1}-\Delta_{t} & =F(x_{t+1})-F(x_{t})\leq\langle\na_{t},x_{t+1}-x_{t}\rangle+\frac{L}{2}\left\Vert x_{t+1}-x_{t}\right\Vert ^{2}\\
 & =-\eta\langle\na_{t},\frac{d_{t}}{\left\Vert d_{t}\right\Vert }\rangle+\frac{\eta^{2}L}{2}=-\eta\left\Vert d_{t}\right\Vert +\eta\langle\xi_{t},\frac{d_{t}}{\left\Vert d_{t}\right\Vert }\rangle+\frac{\eta^{2}L}{2}\\
 & \overset{(a)}{\leq}-\eta\left\Vert d_{t}\right\Vert +\eta\left\Vert \xi_{t}\right\Vert +\frac{\eta^{2}L}{2}\overset{(b)}{\leq}-\eta\left\Vert \na_{t}\right\Vert +2\eta\left\Vert \xi_{t}\right\Vert +\frac{\eta^{2}L}{2}
\end{align*}
where $(a)$ is by $\langle\xi_{t},\frac{d_{t}}{\left\Vert d_{t}\right\Vert }\rangle\leq\left\Vert \xi_{t}\right\Vert \left\Vert \frac{d_{t}}{\left\Vert d_{t}\right\Vert }\right\Vert =\left\Vert \xi_{t}\right\Vert $
and $(b)$ is due to $\left\Vert \na_{t}\right\Vert \leq\left\Vert d_{t}\right\Vert +\left\Vert \xi_{t}\right\Vert $.
Now summing up from $t=1$ to $\tau$, we obtain
\begin{align*}
 & \eta\sum_{t=1}^{\tau}\left\Vert \na_{t}\right\Vert +\Delta_{\tau+1}\leq\Delta_{1}+\frac{\tau\eta^{2}L}{2}+2\eta\sum_{t=1}^{\tau}\left\Vert \xi_{t}\right\Vert \\
\overset{(c)}{=} & \Delta_{1}+\frac{\tau\eta^{2}L}{2}+2\eta\sum_{t=1}^{\tau}\left\Vert \beta^{t}\xi_{0}+\beta\left(\sum_{s=1}^{t}\beta^{t-s}Z_{s}\right)+\left(1-\beta\right)\left(\sum_{s=1}^{t}\beta^{t-s}\epsilon_{s}\right)\right\Vert \\
\leq & \Delta_{1}+\frac{\tau\eta^{2}L}{2}+2\eta\sum_{t=1}^{\tau}\beta^{t}\left\Vert \xi_{0}\right\Vert +\beta\left\Vert \sum_{s=1}^{t}\beta^{t-s}Z_{s}\right\Vert +\left(1-\beta\right)\left\Vert \sum_{s=1}^{t}\beta^{t-s}\epsilon_{s}\right\Vert \\
\leq & \Delta_{1}+\frac{\tau\eta^{2}L}{2}+\frac{2\beta\eta\left\Vert \xi_{0}\right\Vert }{1-\beta}\indi_{\tau\geq1}+2\eta\sum_{t=1}^{\tau}\beta\left\Vert \sum_{s=1}^{t}\beta^{t-s}Z_{s}\right\Vert +\left(1-\beta\right)\left\Vert \sum_{s=1}^{t}\beta^{t-s}\epsilon_{s}\right\Vert \\
\overset{(d)}{\leq} & \Delta_{1}+\frac{\tau\eta^{2}L}{2}+\frac{3\beta\eta\sqrt{L\Delta_{1}}}{1-\beta}\indi_{\tau\geq1}+2\eta\sum_{t=1}^{\tau}\beta\left\Vert \sum_{s=1}^{t}\beta^{t-s}Z_{s}\right\Vert +\left(1-\beta\right)\left\Vert \sum_{s=1}^{t}\beta^{t-s}\epsilon_{s}\right\Vert 
\end{align*}
where we invoke Lemma \ref{lem:represent} in $(c)$; Fact \ref{fact:fact-2}
leads to $2\left\Vert \xi_{0}\right\Vert =2\left\Vert \na_{1}\right\Vert \leq3\sqrt{L\Delta_{1}}$
in $(d)$.
\end{proof}

\subsection{Proof of Lemma \ref{lem:Z-algo}}

\begin{proof}
For Algorithm \ref{alg:algo-1}, from the definition of $Z_{s}$,
we know
\[
\left\Vert Z_{s}\right\Vert =\left\Vert \indi_{s\geq2}\left(\na_{s-1}-\na_{s}\right)\right\Vert \leq\indi_{s\geq2}\left\Vert \na_{s-1}-\na_{s}\right\Vert \leq\indi_{s\geq2}L\left\Vert x_{s-1}-x_{s}\right\Vert \leq\eta L
\]
where we use the smoothness assumption and note that $\left\Vert x_{s-1}-x_{s}\right\Vert =\left\Vert \eta\frac{d_{s}}{\left\Vert d_{s}\right\Vert }\right\Vert =\eta$.
Hence,
\[
\left\Vert \sum_{s=1}^{t}\beta^{t-s}Z_{s}\right\Vert \leq\sum_{s=1}^{t}\beta^{t-s}\left\Vert Z_{s}\right\Vert \leq\sum_{s=1}^{t}\beta^{t-s}\eta L\leq\frac{\eta L}{1-\beta}.
\]

For Algorithm \ref{alg:algo-2}, let $t\in\left[T\right]$ be fixed
and consider $s\in\left[t\right]$. From the definition of $Z_{s}$,
we know $\beta^{t-s}Z_{s}$ is adapted to $\F_{s}$. Besides, 
\[
\E_{s}\left[\beta^{t-s}Z_{s}\right]=\beta^{t-s}\E_{s}\left[\indi_{s\geq2}\left(\nabla f(x_{s},\Xi_{s})-\nabla f(x_{s-1},\Xi_{s})+\na_{s-1}-\na_{s}\right)\right]=0
\]
which implies $\beta^{t-s}Z_{s}$ is a martingale difference sequence.
Note that we have $\left\Vert \beta^{t-1}Z_{1}\right\Vert =0$ and
for $s\geq2$, by the almost surely smoothness assumption,
\[
\left\Vert \beta^{t-s}Z_{s}\right\Vert \leq\left\Vert \nabla f(x_{s},\Xi_{s})-\nabla f(x_{s-1},\Xi_{s})\right\Vert +\left\Vert \na_{s}-\na_{s-1}\right\Vert \leq2L\left\Vert x_{s}-x_{s-1}\right\Vert =2\eta L.
\]
Besides, there is
\begin{align*}
\E_{s}\left[\left\Vert \beta^{t-s}Z_{s}\right\Vert ^{2}\right] & =\beta^{2t-2s}\E_{s}\left[\left\Vert \indi_{s\geq2}\left(\nabla f(x_{s},\Xi_{s})-\nabla f(x_{s-1},\Xi_{s})+\na_{s-1}-\na_{s}\right)\right\Vert ^{2}\right]\\
 & \leq\beta^{2t-2s}\indi_{s\geq2}\E_{s}\left[\left\Vert \nabla f(x_{s},\Xi_{s})-\nabla f(x_{s-1},\Xi_{s})\right\Vert ^{2}\right]\leq\beta^{2t-2s}\eta^{2}L^{2}.
\end{align*}
Now we invoke Lemma \ref{lem:ashok-ineq} to obtain that with probability
at least $1-\delta$, for any $\tau\in\left[t\right]$, there is
\begin{align*}
\left\Vert \sum_{s=1}^{\tau}\beta^{t-s}Z_{s}\right\Vert  & \leq6\eta L\log\frac{3}{\delta}+3\sqrt{\sum_{s=1}^{\tau}\beta^{2t-2s}\eta^{2}L^{2}\log\frac{3}{\delta}}\\
 & \leq3\eta L\left(2\log\frac{3}{\delta}+\sqrt{\frac{\log\frac{3}{\delta}}{1-\beta}}\right)\leq\frac{9\eta L\log\frac{3}{\delta}}{\sqrt{1-\beta}}.
\end{align*}
We choose $\tau=t$ and replace $\delta$ by $\frac{\delta}{T}$ to
finish the proof.
\end{proof}

\subsection{Proof of Lemma \ref{lem:eps-decomp}}

\begin{proof}
Starting with the definition of $\epsilon_{s}=\epsilon_{s}^{b}+\epsilon_{s}^{u}$
to get
\[
\left\Vert \sum_{s=1}^{t}\beta_{1}^{t-s}\epsilon_{s}\right\Vert \leq\left\Vert \sum_{s=1}^{t}\beta^{t-s}\epsilon_{s}^{u}\right\Vert +\left\Vert \sum_{s=1}^{t}\beta^{t-s}\epsilon_{s}^{b}\right\Vert .
\]
Now we define the sequence $U_{s}^{t}$ for $s\in\left\{ 0\right\} \cup\left[t\right]$
\[
U_{s}^{t}=\begin{cases}
\bzero & s=0\\
\sgn\left(\sum_{i=1}^{s-1}U_{i}^{t}\right)\frac{\langle\sum_{i=1}^{s-1}\beta^{t-i}\epsilon_{i}^{u},\beta^{t-s}\epsilon_{s}^{u}\rangle}{\left\Vert \sum_{i=1}^{s-1}\beta^{t-i}\epsilon_{i}^{u}\right\Vert } & s\neq0\text{ and }\sum_{i=1}^{s-1}\beta^{t-i}\epsilon_{i}^{u}\neq\bzero\\
\bzero & s\neq0\text{ and }\sum_{i=1}^{s-1}\beta^{t-i}\epsilon_{i}^{u}=\bzero
\end{cases}
\]
According to Lemma \ref{lem:ashok-decomp}, $U_{s}^{t}$ is a martingale
difference sequence satisfying $\left|U_{s}^{t}\right|\leq\left\Vert \beta^{t-s}\epsilon_{s}^{u}\right\Vert $
and
\begin{align*}
\left\Vert \sum_{s=1}^{t}\beta^{t-s}\epsilon_{s}^{u}\right\Vert  & \leq\left|\sum_{s=1}^{t}U_{s}^{t}\right|+\sqrt{\max_{s\in\left[t\right]}\left\Vert \beta^{t-s}\epsilon_{s}^{u}\right\Vert ^{2}+\sum_{s=1}^{t}\left\Vert \beta^{t-s}\epsilon_{s}^{u}\right\Vert ^{2}}\\
 & \leq\left|\sum_{s=1}^{t}U_{s}^{t}\right|+\sqrt{2\sum_{s=1}^{t}\left\Vert \beta^{t-s}\epsilon_{s}^{u}\right\Vert ^{2}-\E_{s}\left[\left\Vert \beta^{t-s}\epsilon_{s}^{u}\right\Vert ^{2}\right]+\E_{s}\left[\left\Vert \beta^{t-s}\epsilon_{s}^{u}\right\Vert ^{2}\right]}\\
 & \leq\left|\sum_{s=1}^{t}U_{s}^{t}\right|+\sqrt{2\left|\sum_{s=1}^{t}R_{s}^{t}\right|}+\sqrt{2\sum_{s=1}^{t}\E_{s}\left[\left\Vert \beta^{t-s}\epsilon_{s}^{u}\right\Vert ^{2}\right]}
\end{align*}
where $R_{s}^{t}=\left\Vert \beta^{t-s}\epsilon_{s}^{u}\right\Vert ^{2}-\E_{s}\left[\left\Vert \beta^{t-s}\epsilon_{s}^{u}\right\Vert ^{2}\right]$.
The proof is completed now.
\end{proof}

\subsection{Proof of Lemma \ref{lem:U}}

\begin{proof}
For any fixed $t\in\left[T\right]$, note that $U_{s}^{t}$ is a martingale
difference sequence satisfying $\left|U_{s}^{t}\right|\leq\left\Vert \beta^{t-s}\epsilon_{s}^{u}\right\Vert \leq2M$
(Lemmas \ref{lem:eps-bound} and \ref{lem:eps-decomp}). Hence by
Bernstein Iinequality (Lemma \ref{lem:bern}), we know
\[
\Pr\left[\left|\sum_{s=1}^{t}U_{s}^{t}\right|>a\text{ and }\sum_{s=1}^{t}\E_{s}\left[\left(U_{s}^{t}\right)^{2}\right]\leq F\log\frac{4T}{\delta}\right]\leq2\exp\left(-\frac{a^{2}}{2F\log\frac{4T}{\delta}+4Ma/3}\right).
\]
Choose $a>0$ here to satisfy
\[
2\exp\left(-\frac{a^{2}}{2F\log\frac{4T}{\delta}+4Ma/3}\right)=\frac{\delta}{2T}\Rightarrow a=\left(\frac{2}{3}M+\sqrt{\frac{4}{9}M^{2}+2F}\right)\log\frac{4T}{\delta}.
\]
Next, we take $F=\frac{10\sigma^{p}M^{2-p}}{1-\beta}$. Thus, with
probability at least $1-\frac{\delta}{2T}$, there is
\[
\left|\sum_{s=1}^{t}U_{s}^{t}\right|\leq\left(\frac{2}{3}+\sqrt{\frac{4}{9}+\frac{20\left(\sigma/M\right)^{p}}{1-\beta}}\right)M\log\frac{4T}{\delta}\leq\left(\frac{4}{3}+2\sqrt{\frac{5\left(\sigma/M\right)^{p}}{1-\beta}}\right)M\log\frac{4T}{\delta}
\]
or
\[
\sum_{s=1}^{t}\E_{s}\left[\left(U_{s}^{t}\right)^{2}\right]>\frac{10\sigma^{p}M^{2-p}}{1-\beta}\log\frac{4T}{\delta}.
\]
\end{proof}

\subsection{Proof of Lemma \ref{lem:R}}

\begin{proof}
For any fixed $t\in\left[T\right]$, by Lemma \ref{lem:eps-bound},
$R_{s}^{t}=\left\Vert \beta^{t-s}\epsilon_{s}^{u}\right\Vert ^{2}-\E_{s}\left[\left\Vert \beta^{t-s}\epsilon_{s}^{u}\right\Vert ^{2}\right]$
is a martingale difference sequence satisfying 
\[
\left|R_{s}^{t}\right|\leq\left\Vert \beta^{t-s}\epsilon_{s}^{u}\right\Vert ^{2}+\E_{s}\left[\left\Vert \beta^{t-s}\epsilon_{s}^{u}\right\Vert ^{2}\right]\leq8M^{2}.
\]
Hence by Bernstein inequality (Lemma \ref{lem:bern}), we know
\[
\Pr\left[\left|\sum_{s=1}^{t}R_{s}^{t}\right|>a\text{ and }\sum_{s=1}^{t}\E_{s}\left[\left(R_{s}^{t}\right)^{2}\right]\leq F\log\frac{4T}{\delta}\right]\leq2\exp\left(-\frac{a^{2}}{2F\log\frac{4T}{\delta}+16M^{2}a/3}\right).
\]
We choose $a>0$ to satisfy
\[
2\exp\left(-\frac{a^{2}}{2F\log\frac{4T}{\delta}+16M^{2}a/3}\right)=\frac{\delta}{2T}\Rightarrow a=\left(\frac{8}{3}M^{2}+\sqrt{\frac{64}{9}M^{4}+2F}\right)\log\frac{4T}{\delta}.
\]
We take $F=\frac{40\sigma^{p}M^{4-p}}{1-\beta}$ here to obtain with
probability at least $1-\frac{\delta}{2T}$, there is
\[
\left|\sum_{s=1}^{t}R_{s}^{t}\right|\leq\left(\frac{8}{3}+\sqrt{\frac{64}{9}+\frac{80\left(\sigma/M\right)^{p}}{1-\beta}}\right)M^{2}\log\frac{4T}{\delta}\leq\left(\frac{16}{3}+4\sqrt{\frac{5\left(\sigma/M\right)^{p}}{1-\beta}}\right)M^{2}\log\frac{4T}{\delta}
\]
or
\[
\sum_{s=1}^{t}\E_{s}\left[\left(R_{s}^{t}\right)^{2}\right]>\frac{40\sigma^{p}M^{4-p}}{1-\beta}\log\frac{4T}{\delta}.
\]
\end{proof}

\section{Proof of Theorem \ref{thm:thm-algo-1}\label{sec:app-proof-theory}}

The proof of Theorem \ref{thm:thm-algo-1} is almost the same as the
proof of Theorem \ref{thm:thm-algo-2} under our unified analysis
framework.

\begin{proof}
We will use induction to prove that the event $G_{\tau}=E_{\tau}\cap A_{\tau}\cap B_{\tau}$
holds with probability at least $1-\frac{\tau\delta}{T}$ for any
$\tau\in\left\{ 0\right\} \cup\left[T\right]$ where
\[
E_{\tau}=\left\{ \eta\sum_{s=1}^{t}\left\Vert \na_{s}\right\Vert +\Delta_{t+1}\leq2\Delta_{1},\forall t\leq\tau\right\} ;\enskip A_{\tau}=\cap_{t=1}^{\tau}a_{t};\enskip B_{\tau}=\cap_{t=1}^{\tau}b_{t};
\]
and
\begin{align*}
a_{t} & =\left\{ \left|\sum_{s=1}^{t}U_{s}^{t}\right|\leq\left(\frac{4}{3}+2\sqrt{\frac{5\left(\sigma/M\right)^{p}}{1-\beta}}\right)M\log\frac{4T}{\delta}\text{ or }\sum_{s=1}^{t}\E_{s}\left[\left(U_{s}^{t}\right)^{2}\right]>\frac{10\sigma^{p}M^{2-p}}{1-\beta}\log\frac{4T}{\delta}\right\} ;\\
b_{t} & =\left\{ \left|\sum_{s=1}^{t}R_{s}^{t}\right|\leq\left(\frac{16}{3}+4\sqrt{\frac{5\left(\sigma/M\right)^{p}}{1-\beta}}\right)M^{2}\log\frac{4T}{\delta}\text{ or }\sum_{s=1}^{t}\E_{s}\left[\left(R_{s}^{t}\right)^{2}\right]>\frac{40\sigma^{p}M^{4-p}}{1-\beta}\log\frac{4T}{\delta}\right\} .
\end{align*}
Note that $E_{0}=\left\{ \Delta_{1}\leq2\Delta_{1}\right\} $ can
be viewed as the whole probability space. Thus, we have $A_{0}=B_{0}=E_{0}$.
Another useful fact is that $A_{\tau}=A_{\tau-1}\cap a_{\tau}$ and
$B_{\tau}=B_{\tau-1}\cap b_{\tau}$.

Now we can start the induction. When $\tau=0$, $G_{0}=E_{0}=\left\{ \Delta_{1}\leq2\Delta_{1}\right\} $
holds with probability $1=1-\frac{\tau\delta}{T}$. Next, suppose
the induction hypothesis $\Pr\left[G_{\tau-1}\right]\geq1-\frac{(\tau-1)\delta}{T}$
is true for some $\tau\in\left[T\right]$. We will prove that $\Pr\left[G_{\tau}\right]\geq1-\frac{\tau\delta}{T}$
. To start with, we consider the following event 
\[
E_{\tau-1}\cap A_{\tau}\cap B_{\tau}=G_{\tau-1}\cap a_{\tau}\cap b_{\tau}.
\]
From Lemmas \ref{lem:U} and \ref{lem:R}, there are $\Pr\left[a_{\tau}\right]\geq1-\frac{\delta}{2T}$
and $\Pr\left[b_{\tau}\right]\geq1-\frac{\delta}{2T}$. Combining
with the induction hypothesis, we obtain$\Pr\left[E_{\tau-1}\cap A_{\tau}\cap B_{\tau}\right]\geq1-\frac{\tau\delta}{T}.$ 

Note that under the event $E_{\tau-1}$, there is $\Delta_{t}\leq\eta\sum_{s=1}^{t-1}\left\Vert \na_{s}\right\Vert +\Delta_{t}\leq2\Delta_{1}$,
$\forall t\leq\tau$ which implies $\left\Vert \na_{t}\right\Vert \overset{(a)}{\leq}\sqrt{2L\Delta_{t}}\leq2\sqrt{L\Delta_{1}}\overset{(b)}{\leq}\frac{M}{2}$,
$\forall t\leq\tau$ where $(a)$ is by Fact \ref{fact:fact-2} and
for $(b)$ we use $M\geq4\sqrt{L\Delta_{1}}$. Thus, in addition to
$\|\epsilon_{t}^{u}\|\leq2M$, the following two bounds hold for any
$t\leq\tau$ by Lemma \ref{lem:eps-bound}:
\begin{align}
\E_{t}\left[\left\Vert \epsilon_{t}^{u}\right\Vert ^{2}\right] & \le10\sigma^{p}M^{2-p}.\label{eq:u-upper}\\
\left\Vert \epsilon_{t}^{b}\right\Vert  & \le2\sigma^{p}M^{1-p}.\label{eq:b-upper}
\end{align}
Equipped with (\ref{eq:u-upper}) and (\ref{eq:b-upper}), we can
find that for any $t\leq\tau$
\begin{align}
\sum_{s=1}^{t}\E_{s}\left[\left(U_{s}^{t}\right)^{2}\right] & \leq\sum_{s=1}^{t}\E_{s}\left[\left\Vert \beta^{t-s}\epsilon_{s}^{u}\right\Vert ^{2}\right]\leq\sum_{s=1}^{t}\beta^{2t-2s}\cdot10\sigma^{p}M^{2-p}\leq\frac{10\sigma^{p}M^{2-p}}{1-\beta};\label{eq:E-U-alg1}\\
\sum_{s=1}^{t}\E_{s}\left[\left(R_{s}^{t}\right)^{2}\right] & \leq\sum_{s=1}^{t}\E_{s}\left[\left\Vert \beta^{t-s}\epsilon_{s}^{u}\right\Vert ^{4}\right]\leq\sum_{s=1}^{t}\beta^{4t-4s}\cdot4M^{2}\cdot10\sigma^{p}M^{2-p}\leq\frac{40\sigma^{p}M^{4-p}}{1-\beta}.\label{eq:E-R-alg1}
\end{align}
Combining with the definitions of events $a_{t}$ and $b_{t}$, (\ref{eq:E-U-alg1})
and (\ref{eq:E-R-alg1}) imply that under the event $E_{\tau-1}\cap A_{\tau}\cap B_{\tau}$,
the following two bounds hold for any $t\leq\tau$,
\begin{align}
\left|\sum_{s=1}^{t}U_{s}^{t}\right| & \leq\left(\frac{4}{3}+2\sqrt{\frac{5\left(\sigma/M\right)^{p}}{1-\beta}}\right)M\log\frac{4T}{\delta};\label{eq:U-upper-alg1}\\
\left|\sum_{s=1}^{t}R_{s}^{t}\right| & \leq\left(\frac{16}{3}+4\sqrt{\frac{5\left(\sigma/M\right)^{p}}{1-\beta}}\right)M^{2}\log\frac{4T}{\delta}.\label{eq:R-upper-alg1}
\end{align}

Assuming $E_{\tau-1}\cap A_{\tau}\cap B_{\tau}$ holds, we invoke
Lemma \ref{lem:basic-ineq} for time $\tau$ to get
\begin{align*}
 & \eta\sum_{t=1}^{\tau}\left\Vert \na_{t}\right\Vert +\Delta_{\tau+1}\\
\leq & \Delta_{1}+\frac{\tau\eta^{2}L}{2}+\frac{3\beta\eta\sqrt{L\Delta_{1}}}{1-\beta}+2\eta\sum_{t=1}^{\tau}\beta\left\Vert \sum_{s=1}^{t}\beta^{t-s}Z_{s}\right\Vert +\left(1-\beta\right)\left\Vert \sum_{s=1}^{t}\beta^{t-s}\epsilon_{s}\right\Vert \\
\overset{(c)}{\leq} & \Delta_{1}+\frac{\tau\eta^{2}L}{2}+\frac{3\beta\eta\sqrt{L\Delta_{1}}}{1-\beta}+\frac{2\beta\tau\eta^{2}L}{1-\beta}+2\eta\left(1-\beta\right)\sum_{t=1}^{\tau}\left\Vert \sum_{s=1}^{t}\beta^{t-s}\epsilon_{s}\right\Vert \\
\leq & \Delta_{1}+\frac{2\tau\eta^{2}L}{1-\beta}+\frac{3\beta\eta\sqrt{L\Delta_{1}}}{1-\beta}+2\eta\left(1-\beta\right)\sum_{t=1}^{\tau}\left\Vert \sum_{s=1}^{t}\beta^{t-s}\epsilon_{s}\right\Vert \\
\overset{(d)}{\leq} & \Delta_{1}+\frac{2\tau\eta^{2}L}{1-\beta}+\frac{3\beta\eta\sqrt{L\Delta_{1}}}{1-\beta}+40\tau\eta M\left(1-\beta\right)\log\frac{4T}{\delta}\\
\overset{(e)}{\leq} & \Delta_{1}+\frac{\Delta_{1}}{3}+\frac{\Delta_{1}}{3}+\frac{\Delta_{1}}{3}\leq2\Delta_{1}
\end{align*}
where $(c)$ is due to Lemma \ref{lem:Z-algo}. The bound of $\|\sum_{s=1}^{t}\beta^{t-s}\epsilon_{s}\|$
in $(d)$ is totally the same as it in (\ref{eq:2-Delta}) in the
proof of Theorem \ref{thm:thm-algo-2}. We plug in the choice of $\eta=\sqrt{\frac{\left(1-\beta\right)\Delta_{1}}{6TL}}\land\frac{1-\beta}{9\beta}\sqrt{\frac{\Delta_{1}}{L}}\land\frac{\Delta_{1}}{120TM\left(1-\beta\right)\log\frac{4T}{\delta}}$
in $(e)$. Now note this result implies the event $e_{\tau}=\left\{ \eta\sum_{t=1}^{\tau}\left\Vert \na_{t}\right\Vert +\Delta_{\tau+1}\leq2\Delta_{1}\right\} $
is a superset of $E_{\tau-1}\cap A_{\tau}\cap B_{\tau}$. Therefore
\[
\Pr\left[G_{\tau}\right]=\Pr\left[E_{\tau-1}\cap e_{\tau}\cap A_{\tau}\cap B_{\tau}\right]=\Pr\left[E_{\tau-1}\cap A_{\tau}\cap B_{\tau}\right]\geq1-\frac{\tau\delta}{T}.
\]
Hence, the induction is completed.

Finally we know $\Pr\left[E_{T}\right]\geq\Pr\left[G_{T}\right]\geq1-\delta$
which implies with probability at least $1-\delta$, $\eta\sum_{t=1}^{T}\left\Vert \na_{t}\right\Vert +\Delta_{T+1}\leq2\Delta_{1}$.
By plugging our choices of $\beta$, $M$ and $\eta$, we conclude
\begin{align*}
 & \frac{1}{T}\sum_{t=1}^{T}\left\Vert \na_{t}\right\Vert \leq\frac{2\Delta_{1}}{\eta T}\\
= & \frac{2\Delta_{1}}{\sqrt{\frac{T\left(1-\beta\right)\Delta_{1}}{6L}}\land\frac{T\left(1-\beta\right)}{9\beta}\sqrt{\frac{\Delta_{1}}{L}}\land\frac{\Delta_{1}}{120M\left(1-\beta\right)\log\frac{4T}{\delta}}}\\
= & O\left(\frac{\sqrt{L\Delta_{1}}}{\sqrt{T\left(1-\beta\right)}}\lor\frac{\beta\sqrt{L\Delta_{1}}}{T\left(1-\beta\right)}\lor M\log\frac{T}{\delta}\right)\\
= & O\left(\frac{\sqrt{L\Delta_{1}}}{\sqrt{T\left(1-\beta\right)}}\lor\frac{\beta\sqrt{L\Delta_{1}}}{T\left(1-\beta\right)}\lor\sigma\left(1-\beta\right)^{1-1/p}\log\frac{T}{\delta}\lor\sqrt{L\Delta_{1}}\left(1-\beta\right)\log\frac{T}{\delta}\right).\\
= & O\left(\frac{\sqrt{L\Delta_{1}}}{T^{\frac{p-1}{3p-2}}}\lor\frac{\sigma\log\frac{T}{\delta}}{T^{\frac{p-1}{3p-2}}}\lor\frac{\sqrt{L\Delta_{1}}\log\frac{T}{\delta}}{T^{\frac{p}{3p-2}}}\right).
\end{align*}
\end{proof}

\end{document}